\documentclass{article}

\usepackage{arxiv} % version arXiv
\usepackage[utf8]{inputenc} % allow utf-8 input
\usepackage[OT1]{fontenc}    % use 8-bit T1 fonts
\usepackage{natbib}
\usepackage{hyperref}       % hyperlinks
\usepackage{url}            % simple URL typesetting
\usepackage{booktabs}       % professional-quality tables
\usepackage{amsmath}
\usepackage{amssymb}
\usepackage{amsthm}
\usepackage{amsfonts}       % blackboard math symbols
\usepackage{mathrsfs}
\usepackage{dsfont}
\usepackage{nicefrac}       % compact symbols for 1/2, etc.
\usepackage{microtype}      % microtypography
\usepackage{enumerate}
\usepackage{paralist}

\usepackage{color}
\definecolor{gray}{rgb}{.25, .25, .25}
\definecolor{aurometalsaurus}{rgb}{0.43, 0.5, 0.5}
\definecolor{britishracinggreen}{rgb}{0.0, 0.26, 0.15}
\definecolor{burntumber}{rgb}{0.54, 0.2, 0.14}
\definecolor{cobalt}{rgb}{0.0, 0.28, 0.67}
\definecolor{bulgarianrose}{rgb}{0.28, 0.02, 0.03}
\definecolor{ceruleanblue}{rgb}{0.16, 0.32, 0.75}

\title{Risk bounds when learning infinitely many response functions by ordinary linear regression}

\newtheorem{theorem}{Theorem}
\newtheorem{condition}{Condition}
\newtheorem{lemma}{Lemma}
\newtheorem{proposition}{Proposition}
\newtheorem{definition}{Definition}
\newtheorem{corollary}{Corollary}

\author{%
	Vincent Plassier \\
	CMAP, \'Ecole Polytechnique \\
	Institut Polytechnique de Paris \\
	Lagrange Mathematics and Computing Research Center \\
	75007 Paris, France\\
	\texttt{vincent.plassier@ens-paris-saclay.fr}
	\And
	Francois Portier \\
	LTCI, T\'el\'ecom Paris and CREST, ENSAI \\
	Institut polytechnique de Paris \\
	91120 Palaiseau, France \\
	\texttt{francois.portier@gmail.com}
	\And 
	Johan Segers \\
	LIDAM/ISBA, UCLouvain \\
	1348 Louvain-la-Neuve, Belgium \\
	\texttt{johan.segers@uclouvain.be}
}

\newcommand{\reals}{\mathbb{R}}
\newcommand{\Rd}{\reals^d}
\newcommand{\Rdd}{\reals^{d \times d}}
\newcommand{\N}{\mathbb{N}}
\newcommand{\ninf}[1]{\left\|{#1}\right\|_\infty}
\newcommand{\diff}{\mathrm{d}}

\newcommand{\eps}{\varepsilon}
\newcommand{\expec}{\operatorname{\mathbb{E}}}
\newcommand{\T}{\intercal}  % \top
\newcommand{\argmin}{\operatornamewithlimits{\arg\min}}
\newcommand{\oh}{\operatorname{\mathrm{o}}}
\newcommand{\Oh}{\operatorname{\mathrm{O}}}
\newcommand{\pr}{\mathbb{P}}
\newcommand{\abs}[1]{\lvert{#1}\rvert}

\newcommand{\absbigg}[1]{\bigg\lvert{#1}\bigg\rvert}
\newcommand{\tr}{\operatorname{tr}}
\renewcommand{\emptyset}{\varnothing}
\newcommand{\supf}{\sup_{f \in \Fc}}

\newcommand{\Var}{\operatorname{var}}
\newcommand{\var}{\operatorname{var}}

\newcommand{\half}{1/2}  % \nicefrac{1}{2}
\newcommand{\rme}{\mathrm{e}}
\newcommand{\op}{\oh_{\pr}}
\newcommand{\Op}{\Oh_{\pr}}

\newcommand{\prr}[1]{\left({#1}\right)}
\newcommand{\prn}[1]{(\textstyle{#1})}

\newcommand{\br}[1]{\left[{#1}\right]}

\newcommand{\ac}[1]{\left\{{#1}\right\}}
\newcommand{\acn}[1]{\{\textstyle{#1}\}}

\newcommand{\norm}[1]{\lVert{\textstyle{#1}}\rVert}

\newcommand{\Ac}{\mathscr{A}}
\newcommand{\Ec}{\mathscr{E}}
\newcommand{\Fc}{\mathscr{F}}
\newcommand{\Gc}{\mathscr{G}}

\newcommand{\Nc}{\operatorname{\mathscr{N}}}
\newcommand{\Sc}{\mathscr{S}}
\newcommand{\Xc}{\mathscr{X}}
\newcommand{\Yc}{\mathscr{Y}}
\newcommand{\scfunc}{s}

\begin{document}

\maketitle

\begin{abstract}
	Consider the problem of learning a large number of response functions simultaneously based on the same input variables. The training data consist of a single independent random sample of the input variables drawn from a common distribution together with the associated responses. The input variables are mapped into a high-dimensional linear space, called the feature space, and the response functions are modelled as linear functionals of the mapped features, with coefficients calibrated via ordinary least squares. We provide convergence guarantees on the worst-case excess prediction risk by controlling the convergence rate of the excess risk uniformly in the response function. The dimension of the feature map is allowed to tend to infinity with the sample size. The collection of response functions, although potentially infinite, is supposed to have a finite Vapnik--Chervonenkis dimension. The bound derived can be applied when building multiple surrogate models in a reasonable computing time.
\end{abstract}

\section{Introduction}

\paragraph{Context.}

When the outcome of interest is generated by a black box model which cannot be easily evaluated, a well-spread technique is to build a surrogate model allowing to reproduce the behavior of the true model while being computationally cheaper. This approach, known as \textit{response surface}, is popular in many fields of engineering such as \textit{reliability analysis} \citep{bucher1990fast}, \textit{aerospace science} \citep{forrester2009recent}, \textit{energy science} \citep{nguyen2014review}, or electromagnetic dosimetry \citep{azzi2019surrogate}, to name a few. In addition, the response surface methodology is useful in applied mathematics, for instance, in \textit{optimization} when the objective function is difficult to evaluate \citep{jones:2001} and in \textit{Monte Carlo integration}, where a surrogate function can be used to reduce the variance of the Monte Carlo estimate with the help of control variates \citep{portier+s:2019}. For a general presentation of the response surface methodology, we refer to \citet{myers2016response}.

\paragraph{Framework.}

The statistical framework of a response surface is as follows. Consider a real-valued function $f$ defined on a state space $\Xc$, that is, $f:\Xc \to \reals$. In many applications, the function $f$ is a black box and difficult to evaluate.  For instance, a request to $f$ might be obtained by running a heavy computer program. In such a situation, one can afford only a few requests to $f$, that is, one can obtain $f(X_1),\ldots, f(X_n)$, where $n\in \N$ and $X_1,\ldots, X_n$ is an independent sample of $\Xc$-valued random variables with distribution $P$, called the inputs or the covariates. Based on those evaluations, the goal is to build a \textit{surrogate model}, that is, an approximation of $f$ which is easier to calculate than $f$ itself.

Many different methods can be used to build a surrogate model. The simplest one consists in learning $f$ as a linear combination of the covariates by minimizing the sum of squared errors. A well spread extension is to fit a polynomial function instead of a linear one as presented in \citet{myers2016response} or \citet{konakli2016polynomial}.
Since building a response surface consists in the same task as regression, any regression method might be used in principle. Popular methods include moving least-squares \citep{breitkopf2005}, Gaussian processes \citep{frean2008using} or neural nets \citep{bauer2019estimation}. For surrogate models, the approximation method needs to be sufficiently flexible to fit the black box function $f$ well and simple enough to require only a small amount of computations. It is perhaps due to its connection to the regression framework that the problem of building surrogate models has received little specific attention in the statistical learning literature. Our purpose is to address the question of learning many surrogate models simultaneously from a single, random design.

\paragraph{Learning several models simultaneously.}

The fundamental question raised in this paper deals with the ability of building several, possibly infinitely many, surrogate models such that (a) they share the same quality and (b) they are constructed with the help of a single input sample. 
From a theoretical standpoint, it relates to the question of \textit{uniformity} over the tasks: can a broad family of models be learnt with a uniform level of accuracy?  
From a more practical point of view, by working with the same inputs to solve multiple tasks simultaneously, one benefits from a certain computational advantage, as explained below.

Consider a broad class $\Fc$ of black box models $f$. For each such model $f$, a least-squares estimate is obtained out of the linear span of the components of the $d$-dimensional feature map $h = (h_1,\ldots, h_d)^\T : \Xc \to \Rd$, where ${}^\T$ denotes matrix transposition. The feature map $h$ should be known and easy to evaluate. Define the ordinary least squares estimate
\[
	\hat{\beta}_f \in \argmin_{b\in \Rd}  \sum_{i=1}^n \{ f(X_i) -  h(X_i)^\T b \} ^2 .
\]
The surrogate model for $f$ is then defined as $x \mapsto \hat{f}(x) = h(x)^\T \hat{\beta}_f$.

This approach is known as \textit{series estimators} or simply as \textit{least squares estimators} \citep{hardle1990applied, gyorfi2006distribution}. It is quite general as several basis functions might be considered such as polynomials, indicators, spline functions or the Fourier basis. In the regression framework, series estimators have been studied for instance in \citet{newey1997convergence} and \citet{belloni2015some}. Series estimators are a convenient way to include shape constraints on $f$, as for instance when $f$ is partially linear. They also facilitate the computation of derivatives \citep{zhou2000derivative}. In this respect, series estimators can help to build \textit{easy-to-interpret} models, a desirable feature when one is in need of some knowledge about the effects of certain inputs on the black box model $f$.

One motivation for the use of series estimators is the computational advantage they provide when several models are to be learnt at the same time. When a large number $m$ of such surrogate models $f$ are built with different covariates, running $m$ least-squares algorithms, for instance using the Cholesky decomposition, requires $\Oh(m n d^2)$ operations \citep[Section~3.5]{friedman2001elements}, assuming that $d = \Oh(n)$. In our framework of a single training sample, however, the Cholesky decomposition needs to be done only once, and the time needed to compute $m$ least squares estimates is rather $\Oh( n d^2 + m n d)$. 

\paragraph{Uniform convergence rate in random design with increasing dimension.}

For a given model $f \in \Fc$, the error made by a surrogate model $x \mapsto h(x)^\T b$ with coefficient vector $b \in \Rd$ is measured through the $L^2(P)$-risk
\begin{equation*}
	L_f (b) 
	= \int_{\Xc} \{f(x) - h(x)^\T b\}^2 \, \diff P(x) 
	= \expec [ \{ f(X) -  h(X)^\T b \}^2], 
\end{equation*}
where the $\Xc$-valued random variable $X$ has distribution $P$. We place ourselves in the random design setting and study the risk $L_f(\hat{\beta}_f)$ of the least squares estimator $\hat{\beta}_f$. This risk is a random variable whose randomness stems from the one of the training sample $X_1, \ldots, X_n$. 

The main result of the paper concerns the \textit{excess risk} $L_f(\hat{\beta}_f) - \min_{b \in \Rd} L_f(b)$ when $n \to \infty$ and $d\to \infty$, uniformly over $f \in \Fc$. That is, we study the convergence rate to zero of the random variable $\supf \{ L_f(\hat{\beta}_f) -  \min_{b\in\Rd}  L_f(b) \}$. A key quantity is the \emph{leverage function} $q : \Xc \to [0, \infty)$ defined by
\[
	\forall x \in \Xc, \qquad
	q(x) = h(x)^\T  G^{-1}  h(x) ,
	\]
	where $G = \expec[h(X) h(X)^\T]$ is the $d \times d$ Gram matrix of the feature map. 
The leverage function is the population version of the \textit{statistical leverage} of a feature vector $h(X_i)$ in the linear regression model. It plays an important role when analyzing regression with random design \citep{hsu+etal:2014}. Note that $q$ does not change if the feature map is composed with an invertible linear transformation.
Let $\eps_f = f - h^\T \beta_f$ be the error function, with $\beta_f = \argmin_{b \in \Rd} L_f(b)$ the risk-minimizing coefficient vector. Our main result, expressed in Corollary~\ref{cor:simple-rate}, is that
\[
		\sup_{f \in \Fc} \left\{ L_f(\hat{\beta}_f) - \min_{b\in\Rd}  L_f(b) \right\}
		= \Op\left( \frac{\log n}{n} \supf \expec [ q(X)  \eps_{f}^2(X)] \right), \qquad n \to \infty.
\]
Apart from the fact that the obtained bound is invariant under invertible linear transformations of the feature map, this result is remarkable for the three following reasons. First, the dimension $d$ of the feature space is allowed to tend to infinity with the input sample size $n$ at a speed which depends on the leverage function $q$ via Condition~\ref{cond:qn}. Second, in case the class $\mathcal F$ contains only a single response function $f$, a simple analysis leads to a bound for the excess risk that scales as $  \expec [ q(X)  \eps_{f}^2(X)] / n $, see Eq.~\eqref{eq:gamma}, a bound that matches the one of our main result up to a logarithmic term. Third, the quantity $ \expec [ q(X)  \eps_{f}^2(X)] $ takes over the role of the quantity $\sigma_f^2 d$ in the fixed-design setting, where $\sigma_f^2$ is the variance of the error variable in the linear model.

The uniformity in $f \in \Fc$ is achieved by a decomposition of a quadratic form in terms of a sample mean and a U-statistic in combination with a concentration inequality for the suprema of such statistics. The analysis is focused on the ordinary least squares estimator, whereas the extension to ridge regression as in \citet{hsu+etal:2014} is left for further research.

\paragraph{Application to Monte Carlo integration with control variates.}

The past few years, Monte Carlo integration has received increasing interest because of its simplicity and its success facing complex high-dimensional approximation problems. The standard Monte Carlo error has a convergence rate of $1/\sqrt n $, independently of the dimension---see \cite{novak2016} for a review of deeper results around this point. Although a dimension-free convergence rate is comfortable, $1/\sqrt n $ is still relatively slow and some difficulties might arise in situations where we can only make a limited number of requests to the integrand. 
The use of control variates \citep{owen:13,glasserman:2013} has then represented an interesting avenue as it allows to reduce the variance of the standard Monte Carlo estimate without requiring additional evaluations of the integrand. 
Recently, it has been shown \citep{oates+g+c:2017,portier+s:2019} that control variates allow to accelerate the $1/\sqrt n $ convergence rate substantially.
However, whether this acceleration occurs when the error is measured uniformly over a class of integrand functions is, to the best of our knowledge, still unknown.
Motivated by several applications in which uniform results are needed (see Section~\ref{sec:MC} for details), we obtain, as a by-product of the bound in Corollary~\ref{cor:simple-rate}, a uniform convergence rate for control variate Monte Carlo estimates. The fact that the rate established is faster than the standard Monte Carlo rate furnishes an additional argument for the use of control variates in Monte Carlo methods.

\paragraph{Paper outline.}

The mathematical background is presented in Section~\ref{sec:set-up}. The main result and a sketch of its proof are presented in Sections~\ref{sec:thm} and~\ref{sec:proof}, respectively.
The application to control variate Monte Carlo methods is considered in Section ~\ref{sec:MC}. 
Detailed proofs are deferred to the appendices.

\section{Learning multiple response functions simultaneously}
\label{sec:set-up}

\paragraph{Linear model and ordinary least squares estimator.}

Consider a collection $\Fc$ of functions $f : \Xc \to \reals$ on a probability space $(\Xc, \Ac, P)$. We think of $f(x)$ as the real-valued response given an input $x \in \Xc$. Given an independent random sample $X_1, \ldots, X_n$ from $P$ together with the associated responses $f(X_1), \ldots, f(X_n)$ for every $f \in \Fc$, we wish to learn the values $f(x)$ of the response functions $f \in \Fc$ for new but yet unobserved inputs $x \in \Xc$. To this end, we map the input space $\Xc$ into a feature space $\Rd$ via a feature map $h : \Xc \to \Rd : x \mapsto h(x) = (h_1(x), \ldots, h_d(x))^\T$. One of the feature functions $h_j$ could be the constant function $1$, corresponding to an intercept. The response functions are modelled as linear functionals of the mapped features with coefficients estimated by ordinary least squares. The approximation to the response function $f \in \Fc$ is thus
\[
	\forall x \in \Xc, \qquad
	\hat{f}(x) = h(x)^\T \hat{\beta}_f 
	\quad \text{ where } \quad \hat{\beta}_f = \argmin_{b \in \Rd} \sum_{i=1}^n \{f(X_i) - h(X_i)^\T b\}^2.
\]
We wish to control the learning error $\hat{f} - f$ uniformly in $f \in \Fc$.

Sharing the same inputs and the same feature map across multiple response functions brings computational gains. Classical least-squares theory yields 
\[ 
	\forall f \in \Fc, \qquad
	\hat{\beta}_f 
	= G_n^{-1} \frac{1}{n} \sum_{i=1}^n h(X_i) f(X_i)
	\quad \text{ where } \quad G_n = \frac{1}{n} \sum_{i=1}^n h(X_i) h(X_i)^\T.
\]
In case the empirical Gram matrix $G_n$ is not invertible, a pseudo-inverse is used instead. It follows that the predicted responses are linear in the observed responses:
\[
	\forall f \in \Fc, x \in \Xc, \qquad 
	\hat{f}(x) = \frac{1}{n} \sum_{i=1}^n w(x, X_i) f(X_i)
	\quad \text{ where } \quad w(x, X_i) = h(x)^\T G_n^{-1} h(X_i).
\] 
The weights $w(x, X_i)$ do not depend on the response function $f$. This invariance is a computational advantage if multiple response functions $f \in \Fc$ are to be learned simultaneously. 

\paragraph{Worst-case excess prediction risk.}

The response functions are modelled as linear functionals of the mapped features via
\begin{equation}
\label{eq:linmod}
	\forall f \in \Fc, \forall x \in \Xc, \qquad
	f(x) = h(x)^{\T} \beta_f + \eps_f(x).
\end{equation}
The coefficient vector $\beta_f$ is defined as the minimizer over $b \in \Rd$ of the prediction risk
\[
	\forall f \in \Fc, \forall b \in \Rd, \qquad
	L_f(b) = \expec[\{f(X) - h(X)^\T b\}^2].
\]
Here, the expectation is taken with respect to a random feature $X$ with distribution $P$ and it is assumed that $f$ and $h$ have finite second moments. Let $G = \expec[h(X) h(X)^\T]$ denote the $d \times d$ Gram matrix of the feature map, assumed to be positive definite. Classical least squares theory yields
\[
	\forall f \in \Fc, \qquad
	\beta_f = \argmin_{b \in \Rd} L_f(b) = G^{-1} \expec[h(X) f(X)].
\]

Since $\eps_f(X) = f(X) - h(X)^\T \beta_f$ is orthogonal to $h(X)$, i.e., $\expec[h(X) \eps_f(X)] = 0$, the \textit{excess risk} associated to any other coefficient vector $b \in \Rd$ is
\[
	L_f(b) - L_f(\beta_f) = \expec[\{h(X)^\T(b - \beta_f)\}^2] = (b - \beta_f)^\T G (b - \beta_f).
\]

The optimal coefficient vector $\beta_f$ is unknown, so we estimate it by the least-squares estimator $\hat{\beta}_f$. The expected squared error made by the approximated response function for a new input distributed according to $P$ is
\begin{equation}
\label{eq:Lfdecomp}
	\int_{\Xc} \{\hat{f}(x) - f(x)\}^2 \, \diff P(x)
	=
	L_f(\hat{\beta}_f)
	=
	L_f(\beta_f) + (\hat{\beta}_f - \beta_f)^\T \, G \, (\hat{\beta}_f - \beta_f).
\end{equation}
The right-hand side of \eqref{eq:Lfdecomp} decomposes the expected squared error into two parts. 
\begin{itemize}
	\item The first term, $L_f(\beta_f)$, is deterministic. It represents the \emph{modelling error} stemming from the linear model in~\eqref{eq:linmod}. Given the model, this term is incompressible and does not depend on the learning algorithm nor on the training data. 
	\item The second term on the right-hand side in \eqref{eq:Lfdecomp} is random. It represents the \emph{learning error} due to the estimation step and the randomness of the training data. 
\end{itemize}

It is on the learning error that we focus our analysis, with the particularity that we consider the error uniformly in the response function. Our object of interest is thus the \textit{worst-case excess prediction risk}
\begin{equation}
\label{eq:Lfworst}
	\sup_{f \in \Fc} \{L_f(\hat{\beta}_f) - L_f(\beta_f)\}
	=
	\sup_{f \in \Fc} (\hat{\beta}_f - \beta_f)^\T \, G \, (\hat{\beta}_f - \beta_f).
\end{equation}
We are interested in the rate at which this supremum tends to zero as the sample size $n$ and the dimension $d$ of the feature map tend to infinity. 

It is to be emphasized that our setting is that of a random design. The excess prediction risk $L_f(\hat{\beta}_f) - L_f(\beta_f)$ is a nonnegative random variable that is constructed out of the random training sample $X_1, \ldots, X_n$. As is clear from \eqref{eq:Lfdecomp}, it incorporates the risk associated to a new and yet unobserved input $x \in \Xc$, averaged over $P$. The expression in \eqref{eq:Lfworst} is thus a supremum over potentially infinitely many random variables, each variable being built on the same inputs.

\paragraph{Excess prediction risk of a single response.}

Fix a response function $f \in \Fc$. Let $P_n$ denote the empirical distribution of the training sample $X_1, \ldots, X_n$, assigning probability $1/n$ to each observed input. For a real-valued, vector-valued or matrix-valued function $g$ on $\Xc$, expectations with respect to the unknown sampling distribution $P$ and the empirical distribution $P_n$ are denoted respectively by
\begin{align*}
	P(g) &= \expec[g(X)] = \int_{\Xc} g(x) \, \diff P(x), &
	P_n(g) &= \frac{1}{n} \sum_{i=1}^n g(X_i).
\end{align*}
Both operators are linear: for instance, $P(Ag) = A P(g)$ and $P_n(A g) = A P_n(g)$ if $A$ is a linear map between Euclidean spaces of suitable dimension. Using this operator notation, we get
\begin{align*}
	G &= P(h h^\T), & \beta_f &= G^{-1} P(h f), \\
	G_n &= P_n(h h^\T), & \hat{\beta}_f &= G_n^{-1} P_n(h f).
\end{align*}
Since $f = h^\T \beta_f + \eps_f$, the estimated coefficient vector is
\[
	\hat{\beta}_f 
	= G_n^{-1} P_n[h (h^\T \beta_f + \eps_f)] 
	= \beta_f + G_n^{-1} P_n(h \eps_f).
\]
The excess prediction risk is thus
\begin{equation}
\label{eq:Lfrepres0}
	L_f(\hat{\beta}_f) - L_f(\beta_f)
	=
	P_n(h \eps_f)^\T G_n^{-1} G G_n^{-1} P_n(h \eps_f).
\end{equation}

The predicted response functions $\hat{f}$ are linear combinations of the components $h_1, \ldots, h_d$ of the feature map $h$. They only depend on the feature map $h$ through the linear span of the functions $h_1, \ldots, h_d$. Therefore, the predicted responses remain unchanged if we compose the feature map with an invertible linear transformation $A$ of $\Rd$. Let $G^{\half}$ be the unique symmetric square root matrix of $G$ and let $G^{-\half}$ be its inverse. The whitened feature map is $\hbar = G^{-\half} h : \Xc \to \Rd$. If the random input $X$ has distribution $P$, then $\expec[ \hbar(X) \hbar(X)^\T ] = P(\hbar \hbar^\T) = I_d$, the $d \times d$ identity matrix. The empirical Gram matrix of the whitened feature map is
\[
	P_n(\hbar \hbar^\T) = G^{-\half} P_n(h h^\T) G^{-\half} = G^{-\half} G_n G^{-\half}.
\]
Since $h = G^{\half} \hbar$ and $P_n(\hbar \hbar^\T)^{-1} = G^{\half} G_n^{-1} G^{\half}$, the excess prediction risk in \eqref{eq:Lfrepres0} becomes
\begin{equation}
\label{eq:Lfrepres}
	L_f(\hat{\beta}_f) - L_f(\beta_f)
	=
	\abs{ P_n(\hbar \hbar^\T)^{-1} P_n(\hbar \eps_f) }_2^2,
\end{equation}
where $\abs{y}_2 = (y^\T y)^{\half}$ denotes the Euclidean norm of a vector $y \in \Rd$. 

Since $P(\hbar \hbar^\T) = I_d$, it is reasonable to expect that $P_n(\hbar \hbar^\T)^{-1}$ is approximately equal to $I_d$, at least if $n$ is large and $d$ is not too large compared to $n$; see Lemma~\ref{lem:Gninv} below for a precise statement. In that case, the excess prediction risk in \eqref{eq:Lfrepres} is approximately equal to $\abs{ P_n(\hbar \eps_f) }_2^2$. The expectation of the latter random variable can be easily calculated: since $P(\hbar \eps_f) = 0$, the terms with $i \ne j$ in the double sum below vanish and we find
\begin{equation}
\label{eq:gamma}
	\expec\left[ \abs{ P_n(\hbar \eps_f) }_2^2 \right]
	= \frac{1}{n^2} \sum_{i=1}^n \sum_{j=1}^n \expec[\eps_f(X_i) \hbar(X_i)^\T \hbar(X_j) \eps_f(X_j)]
	= \frac{1}{n} P(\hbar^\T \hbar \eps_f^2).
\end{equation}
The excess prediction risk for a single response function $f$ can thus be expected to have an order of magnitude equal to $n^{-1} P(\hbar^\T \hbar \eps_f^{2})$. In comparison, in the fixed-design case, where the training sample is considered as non-random, the expected excess risk is equal to $\sigma_f^2 d / n$, where $\sigma_f^2$ is the error variance \citep{hsu+etal:2014}. Eq.~\eqref{eq:gamma} motivates why in Theorem~\ref{thm:main}, the convergence rate for the worst-case prediction risk over the whole response family $\Fc$ involves the quantity $\supf P(\hbar^\T \hbar \eps_f^2)$.

\section{Convergence rate of the worst-case excess prediction risk}
\label{sec:thm}

% !TEX root = risk-bounds.tex

\subsection{Notation}

Consider an asymptotic setting where the size $n$ of the training sample tends to infinity. The feature map may change with $n$: with a slight change of notation, we write henceforth 
\[ 
	h_n = (h_{n,1},\ldots,h_{n,d_n})^\T : \Xc \to \reals^{d_n}. 
\]
The feature dimension $d_n \ge 1$ depends on $n$ and may tend to infinity. The whitened feature map is 
\[
	\hbar_n = P(h_n h_n^\T)^{-\half} h_n : \Xc \to \reals^{d_n},
\]
where the $d_n \times d_n$ Gram matrix $P(h_n h_n^\T)$ is supposed to be invertible; otherwise, we can omit some components $h_{n,j}$ without affecting the linear span of the component functions. The least squares coefficient vectors and modelling errors depend on $n$ as well: we write
\[
	\forall f \in \Fc, \forall x \in \Xc, \qquad
	f(x) = h_n(x)^\T \beta_{n,f} + \eps_{n,f}(x)
	\quad \text{ where } \quad \beta_{n,f} = P(h_n h_n^\T)^{-1} P(h_n f).
\]
The least squares estimator of $\beta_{n,f}$ is $\hat{\beta}_{n,f} = P_n(h_n h_n^\T)^{-1} P_n(h_n f)$.

In this setting, the \emph{leverage function} $q_n : \Xc \to [0, \infty)$ is defined by
\[
	\forall x \in \Xc, \qquad
	q_n(x) 
	= h_n(x)^\T P(h_n h_n^\T)^{-1} h_n(x) 
	= \hbar_n(x)^\T \hbar_n(x)
	= \abs{ \hbar_n(x) }_2^2.
\]
The name of $q_n$ is derived from the notion of leverage of a design point in multiple linear regression. Note that $q_n$ does not change if the feature map is composed with an invertible linear transformation. We always have $P(q_n) = \tr[ P(\hbar_n \hbar_n^\T) ] = d_n$, where $\tr(A)$ denotes the trace of a square matrix $A$. 

Let $\pr$ denote the probability measure on the probability space on which the random inputs $X_1, \ldots, X_n$, taking values in $\Xc$, are defined. For any sequence $(Y_n)_n$ of real-valued random variables on that space and for any positive sequence $(a_n)_n$, the expression $Y_n = \Op(a_n)$ as $n \to \infty$ signifies that $Y_n / a_n$ is bounded in probability, that is, for every $\epsilon > 0$ there exists $K > 0$ such that $\limsup_{n \to \infty} \pr(|Y_n| > a_n K) \le \epsilon$. Similarly, the expression $Y_n = \op(a_n)$ as $n \to \infty$ signifies that $Y_n / a_n$ converges to zero in probability, that is, $\lim_{n \to \infty} \pr(\abs{Y_n} > a_n \epsilon) = 0$ for every $\epsilon > 0$. Our aim is to determine a positive sequence $a_n$ such that $a_n \to 0$ and, under reasonable assumptions,
\[
	\sup_{f \in \Fc} \left\{ L_f(\hat{\beta}_{n,f}) - L_f(\beta_{n,f}) \right\} = \Op(a_n), \qquad n \to \infty.
\]
Lastly, the supremum norm of a function $g : \Xc \to \reals$ is denoted by $\ninf{g} = \sup_{x \in \Xc} \abs{g(x)}$.

\subsection{Conditions} 

The conditions under which the main result holds concern the leverage function $q_n$ and the response functions $\Fc$.

\begin{condition}
	\label{cond:qn}
	One of the two alternative conditions hold:
	\begin{enumerate}[(a)]
		\item $P(q_n^2) = \oh(n)$ and $\log(\ninf{q_n}) = \Oh(\log n)$ as $n \to \infty$;
		\item $\ninf{q_n} \log(2 d_n) = \oh(n)$ as $n \to \infty$.
	\end{enumerate}
\end{condition}

Since $P(q_n) = d_n$ and $P(q_n^2) \ge [P(q_n)]^2$, condition (a) implies that $d_n = \oh(n^{\half})$ as $n \to \infty$. As $P(q_n^2) \le \ninf{q_n} P(q_n) = \ninf{q_n} d_n$, a sufficient condition for (a) moreover is that $\ninf{q_n} = \oh(n/d_n)$, which is the leverage condition in \citet{portier+s:2019}. Here, we have just assumed that the speed at which $\ninf{q_n}$ tends to infinity is at most polynomial in $n$. Condition (b) implies that $d_n \log(2 d_n) = \oh(n)$ as $n \to \infty$ but, compared to (a), imposes a stronger condition on $\ninf{q_n}$.

\begin{condition}
	\label{cond:F}
	The collection $\Fc$ of response functions admits a uniformly bounded envelope $F : \Xc \to [0, \infty)$, i.e., $\abs{f(x)} \le F(x)$ for any $f \in \Fc$ and any $x \in \Xc$, and $\ninf{F}$ is finite. In addition, $\Fc$ is supposed to be at most countably infinite.
\end{condition}

As $\ninf{q_n}$ and $\ninf{F}$ are finite, the collection of error functions $\{ \eps_{n,f} : f \in \Fc \}$ is uniformly bounded, see Lemma~\ref{lemma:maj M_n} in Appendix~\ref{supp:lemmas}.

The assumption in Condition~\ref{cond:F} that $\Fc$ is countable assures that suprema over random variables indexed by $f \in \Fc$ are measurable. Otherwise, probabilities involving such suprema would need to be replaced by outer probabilities \citep[Part~1]{vdvaart+w:1996}. In practice, the countability assumption is harmless insofar as $\Fc$ can usually be approximated by a countable dense subfamily anyway without affecting the value of the supremum \citep[Section~2.3.3]{vdvaart+w:1996}.

Covering numbers capture the complexity of a subset of a metric space and play a central role in a number of areas in information theory and statistics, including nonparametric function estimation, density estimation, empirical processes, and machine learning.

\begin{definition}[Covering number]
	For a subset $\Fc$ of a metric space $(\Yc, \rho),$ the $\eta$-covering number $\Nc(\Fc, \rho, \eta)$ is the smallest number of open $\rho$-balls of radius $\eta > 0$ required to cover $\Fc$, i.e.,
	\begin{align*}
		\Nc(\Fc, \rho, \eta) = \min\left\{p \ge 1: \ \exists f_1,\ldots,f_p \in \Yc, \; \Fc \subset \bigcup_{i=1}^p B_{\rho}(f_i,\eta)\right\},
	\end{align*}
	where $B_{\rho}(f,\eta) = \left\{ g \in \Yc : \rho(g, f) < \eta \right\}$ for $f \in \Yc$ and $\eta > 0$.
\end{definition}

For the definition of Vapnik--Chervonenkis (VC) classes, we follow \citet{gine+g:1999}.

\begin{definition}[VC-class]\label{def:VC}
	A class $\Fc$ of real functions on a measurable space $(\Xc, \Ac)$ is called a VC-class of parameters $(v, A) \in (0, \infty) \times [1, \infty)$ with respect to the envelope $F$ if for any $0 < \eta < 1$ and any probability measure $Q$ on $(\Xc, \Ac)$, we have
	\begin{equation*}
	\Nc\left(\Fc, L^2(Q), \eta \|F\|_{L^2(Q)}\right) \le (A/\eta)^{v}.
	\end{equation*}
\end{definition}

In Definition~\ref{def:VC}, we view $\Fc$ as a subset of the metric space $L^2(Q) \equiv L^2(\Xc, \Ac, Q)$ of $Q$-square-integrable functions $f : \Xc \to \reals$ equipped with the metric $\rho(f, g) = \norm{f - g}_{L^2(Q)}$, where $\norm{h}_{L^2(Q)} = [Q(h^2)]^{\half}$ for measurable $h : \Xc \to \reals$.

\begin{condition}
	\label{cond:VC}
	With respect to the envelope $F$, the collection $\Fc$ is VC with parameters $(v, A)$.
\end{condition}

\subsection{Main result}

The maximal error is
\begin{equation}
\label{eq:Mn}
	M_n = \supf \ninf{\eps_{n,f}}
\end{equation}
and the growth rate of the worst-case excess prediction risk will be expressed in terms of
\begin{equation}
\label{eq:Ln}
	\gamma_n^2 = \supf P(q_n \eps_{n,f}^2)
	\qquad \text{and} \qquad
	L_n^2 = M_n^2 \ninf{q_n}.
\end{equation}
Clearly, $\gamma_n^2 \le L_n^2$. For $a, b \in \reals$, write $a \vee b = \max(a, b)$ and $a \wedge b = \min(a, b)$. The positive part of $a \in \reals$ is $(a)_+ = a \vee 0$.

\begin{theorem}[Convergence rate of worst-case excess prediction risk]
	\label{thm:main}
	If Conditions~\ref{cond:qn}, \ref{cond:F}, and~\ref{cond:VC} hold, then, as $n \to \infty$,
	\[
		\supf \left\{L_f(\hat{\beta}_f) - L_f(\beta_f)\right\}
		= \Op \left( \left\{ \gamma_n^2 \vee \left(\tau_n  r_n^{1/2}\right) \right\} r_n \right),
		\qquad n \to \infty,
	\]
	where
	\begin{align*}
		\tau_n &= L_n^2 \left(1 \vee (a_n/M_n)\right) \qquad \text{and} \\
		r_n &= 1 \wedge \left[ \frac{\log n}{n} \left\{
			1 + \left( \frac{(\log M_n^{-1})_+}{\log n} \right)^{3/2} 
			\right\} \right]
	\end{align*}
	and where $a_n$ is defined in \eqref{eq:an} and is of the order $\oh\prn{\exp\acn{-n^{2/3}}}$ as $n \to \infty$.
\end{theorem}

When all the functions $f\in\Fc$ are in the vector space $\acn{\beta^{\T}h_n: \beta\in\reals^{d_n}}$, the maximal error becomes $M_n=0$ and therefore we find that the worst-case excess prediction risk is equal to zero.
The quantity $\gamma_n^2$ is related to the $L^2$-norm of the functions $q_n \eps_{n,f}^2$ while $L_n^2$ is related to their supremum norm. It is reasonable to hope that the latter will not be too large in comparison to the former. This motivates the assumptions in the next corollary.

\begin{corollary}[Simplified rates]
	\label{cor:simple-rate}
	If $(\log M_n^{-1})_+ = \Oh(\log n)$ as $n \to \infty$, i.e., if there is some $\alpha > 0$ such that $\liminf_{n \to \infty} n^\alpha M_n > 0$, then
	\begin{equation}
	\label{eq:cor:simpler-rate}
		\supf \left\{ L_f(\hat{\beta}_f) - L_f(\beta_f) \right\}
		= \Op \left(
			\left\{
				\gamma_n^2 \vee \left( L_n^2 \sqrt{\frac{\log n}{n}} \right)
			\right\} \frac{\log n}{n}
		\right), \qquad n \to \infty.
	\end{equation}
	If, moreover, $L_n^2 = \Oh \left( \gamma_n^2 \sqrt{n / \log n} \right)$, then
	\begin{equation}
	\label{eq:cor:simple-rate}
		\supf \left\{ L_f(\hat{\beta}_f) - L_f(\beta_f) \right\}
		= \Oh \left( \gamma_n^2 \frac{\log n}{n} \right),
		\qquad n \to \infty.
	\end{equation}
\end{corollary}

If, on the other hand, the sequence with general term $(\log M_n^{-1})_+$ is of larger order than $\log n$, then for every function $f\in\Fc$, the sequence $(\ninf{\eps_{n,f}})_{n\in\N}$ converges very quickly to zero. This corresponds to the ``ideal case'' where the family $\Fc$ is well approximated by the chosen regressors $(h_{n,1},\ldots,h_{n,d_n})$.
In this case, $(M_n)$ converges to zero faster than $(n^{-\alpha})$ for any $\alpha > 0$ and the general form of the rate in Theorem~\ref{thm:main} implies that the worst-case excess prediction risk converges to zero very fast.

Apart from the factor $\log n$, the convergence rate in \eqref{eq:cor:simple-rate} corresponds to the one for a single response function $f$ as discussed in the paragraph around Eq.~\eqref{eq:gamma}. The additional factor $\log n$ stems from a concentration inequality for U-statistics in combination with bounds on the covering numbers of function classes derived from $\Fc$.

\section{Sketch of proof of Theorem~\ref{thm:main}}
\label{sec:proof}

% !TEX root = risk-bounds.tex

Let $\lambda_{\min}(A)$ and $\lambda_{\max}(A)$ denote the smallest and largest eigenvalue, respectively, of the symmetric matrix $A$. Consider the matrix norm $\abs{A}_2 = \sup \{ \abs{A y}_2 : y \in \Rd, \abs{y}_2 \le 1 \}$ for $A \in \Rdd$. If $A$ is symmetric and positive semi-definite, then $\abs{A}_2 = \lambda_{\max}(A)$. If, moreover, $A$ is positive definite, then $\lambda_{\max}(A^{-1}) = \{\lambda_{\min}(A)\}^{-1}$. In view of \eqref{eq:Lfrepres}, the worst-case excess prediction risk is bounded by
\begin{equation}
\label{eq:worstLfbound}
	\supf \left\{ L_{f}(\hat{\beta}_{f})-L_{f}(\beta_{f}) \right\}
	\le \left\{ \lambda_{\min}\bigl( P_{n}(\hbar_n \hbar_n^{\T}) \bigr) \right\}^{-2}
	\cdot \supf\left|P_{n}(\hbar_n \eps_{n,f})\right|_{2}^{2}.
\end{equation}
Under reasonable conditions permitting $d_n \to \infty$, the smallest eigenvalue of $P_{n}(\hbar_n \hbar_n^{\T})$ remains bounded away from zero with high probability.

\begin{lemma}
\label{lem:Gninv}
Suppose one of the following two conditions holds:
\begin{compactenum}[(a)]
\item $P(q_n^2) = \oh(n)$ as $n \to \infty$;
\item $\ninf{q_n} \log(2 d_n) = \oh(n)$ as $n \to \infty$.
\end{compactenum}
Then $P_{n}(\hbar_{n} \hbar_{n}^\T)$ is invertible with probability tending to one and $\lambda_{\min} \{ P_{n}(\hbar_n \hbar_n^{\T}) \} \ge 1 + \op(1)$ as $n \to \infty$.
\end{lemma}

The proof of Lemma~\ref{lem:Gninv} in case (a) builds upon \citet[Lemmas~2 and~3]{portier+s:2019}, while the one in case (b) is based upon \citet[Lemma~A.2]{leluc2019control}, relying on a matrix Chernoff inequality due to \citet[Theorem~5.1.1]{tropp2015introduction}. The proof is given in Section~\ref{supp:lemmas} in the appendices.

In view of the bound \eqref{eq:worstLfbound} in combination with Lemma~\ref{lem:Gninv}, it is sufficient to show the claimed convergence rate with the excess risk $L_f(\hat{\beta}_n) - L_f(\beta_f)$ replaced by $\abs{ P_n( \hbar_n \eps_{n,f} ) }_2^2$. For $f \in \Fc$, define $g_{n,f} : \Xc^2 \to \reals$ by
\[
	\forall (x, y) \in \Xc^2, \qquad
	g_{n,f}(x, y) = \eps_{n,f}(x) \hbar_n(x)^\T \hbar_n(y) \eps_{n,f}(y).
\]
Note that $g_{n,f}(x, x) = q_n(x) \eps_{n,f}^2(x)$ for $x \in \Xc$. 
The quantity of interest is bounded by
\begin{align}\label{eq:n2Pndecomp}
	n^2 \abs{ P_n( \hbar_n \eps_{n,f} ) }_2^2
	\le
	n P_n (q_n \eps_{n,f}^2) 
	+ \left| \sum_{1 \le i \ne j \le n} g_{n,f}(X_i, X_j) \right|.
\end{align}
We will bound the supremum over $f \in \Fc$ of the sum on the right-hand side of \eqref{eq:n2Pndecomp} by the sum of the suprema of the two terms.
\begin{itemize}
	\item Using concentration inequalities established in \citet{talagrand1994sharper} and \citet{gine2001consistency}, we show that the first supremum is dominated by the stated convergence rate.
%	\item The second supremum is a supremum over a sum of independent and identically distributed random variables with mean zero and finite variance. We will find its rate of convergence to zero via Proposition~\ref{prop:general term_Pn f} below.
	\item The second supremum involves a U-statistic of order two with a degenerate kernel: by orthogonality of $h_n$ and $\eps_{n,f}$, we have $\expec \left[ g_{n,f}(X, x) \right] = \expec \left[ \eps_{n,f}(X) \hbar_n(X)^\T \right] \hbar_n(x) \eps_{n,f}(x) = 0$ and similarly $\expec \left[ g_{n,f}(x, X) \right] = 0$ for every $x \in \Xc$. We determine its convergence rate via a concentration inequality due to \citet{major2006estimate} quoted as Theorem~\ref{prop:U} in Appendix~\ref{supp:major}.

\end{itemize}

To deal with suprema over $f \in \Fc$, % of the right-hand side of \eqref{eq:n2Pndecomp},
we will need to control the covering numbers of the classes $\Gc_n$ and $\Gc_n^{(d)}$ (``\emph{d}'' for ``diagonal'') given by
\begin{align}
\label{eq:Gn}
	\Gc_n &= \{ g_{n,f} : f \in \Fc \}, &
	\Gc_n^{(d)} &= \{  q_n^{1/2}\eps_{n,f} : f \in \Fc \}.
\end{align}
We will find bounds on their covering numbers in terms of those of the ones of the collection of response functions $\Fc$. The bounds are of potentially independent interest. The proof of Proposition~\ref{prop:cover} is given in Appendix~\ref{supp:prop:cover}.

\begin{proposition}[Preservation VC-class] 
	\label{prop:cover}
	Let $\Fc$ be a VC-class with parameters $(v, A)$ with respect to the envelope function $F$. Assume that the associated residuals $\eps_{n,f}$ are uniformly bounded, i.e., there exists $M_n>0$ such that $\sup_{f\in \Fc}\ninf{\eps_{n,f}} \le M_n$. Then $\Gc_n$ and $\Gc_n^{(d)}$ defined in~\eqref{eq:Gn} are VC-classes with respect to the envelopes $M_n^2 \ninf{q_n}$ and $M_n\ninf{q_n}^{1/2}$ with parameters $(4v, 4A_n)$ and $(2v, A_n)$ respectively, where
	\begin{equation} 
	\label{eq:def A_n}
	A_n = 8 A \ninf{F} \ninf{q_n}^{\half} / {M_n}.
%	8 A\ninf{F} \ninf{q_n}/{M_n} .
	\end{equation}
\end{proposition}

Following the above plan, the proof of Theorem~\ref{thm:main} is given in detail in Appendix~\ref{supp:thm:main}.

\section{Uniform bound for Monte Carlo integration with control variates}\label{sec:MC}
% !TEX root = risk-bounds.tex

This section investigates the application of the previous results to Monte Carlo estimates constructed with the help of control variates in order to reduce the variance.
We start by presenting several applications in which uniform bounds for Monte Carlo methods are of interest. Then we give the mathematical background and finally provide sharp uniform error bounds for Monte Carlo estimates that use control variates.

\subsection{Uniformity in Monte Carlo procedures}

Proving uniform bounds on the error of Monte Carlo methods is motivated by the following three applications.

\paragraph{Latent variable models.}

Suppose we are interested in estimating the distribution of the variable $X$ whose density is assumed to lie in the model
\[
p_\theta (y) = \int p_\theta (y|z) \, p(z) \, \diff z
\]
where $\theta\in \Theta$ is the parameter to estimate, $p(z)$ is the known density of the so-called latent variable and $ \{(y,z)\mapsto p_\theta (y|z)\}_{\theta\in \Theta}$ is a model of conditional densities ($Y$ conditionally on the latent variable). This is actually a frequent situation in economics \citep{mcfadden2001economic} and medicine \citep[example 4, 6 and 9]{mcculloch2005generalized}. Given independent and identically distributed random variables $Y_1,\ldots,Y_n$ observed from the previous model with parameter $\theta_0$, the log-likelihood function takes the form
\[
\theta \mapsto \sum_{i=1}^n\log\left(  \int p_\theta(Y_i|z) \, p(z) \, \diff z\right).
\]
In most cases, each term in the previous sum is intractable and the approach proposed in \citet{mcfadden1994estimation} consists in replacing the unknown integrals by Monte Carlo estimates. In such a procedure, there is an additional estimation error compared to the statistical error of the standard maximum likelihood estimator. This additional error can be controlled in terms of the approximation error of $\int p_\theta(y|z) \, p(z) \, \diff z$ by the Monte Carlo estimate uniformly in $(y, \theta)$.

\paragraph{Stochastic programming.}

Consider the stochastic optimization problem
\[
	\min_{\theta\in \Theta} F(\theta) \qquad \text{with} \qquad F(\theta) = \expec [ f(\theta, X) ],
\]
where $X$ is a random variable in some space $\Xc$ with distribution $P$ and where $\Theta$ is a Euclidean set. The response functions are thus the maps $x \mapsto f(\theta, x)$ as $\theta$ ranges over $\Theta$. This problem is different from standard optimization because it takes into account some uncertainty in the output of the function $f$. One might think of the following toy example: $f$ is the output of a laboratory experiment, e.g., the amount of salt in a solution,  $\theta$ is the input of the experiment, e.g., the temperature of the solution and $X$ gathers unobserved random factors that influence the output. Optimizing such kind of functions is of interest in many different fields and we refer the reader to \citet{shapiro2014lectures} for more concrete examples (see also the example below that deals with quantile estimation). 
This problem might be solved using two competitive approaches: the sample average approximation (SAA) and stochastic approximation techniques such as gradient descent---see \citet{nemirovski2009robust} for a comparison between both approaches.
The SAA approach follows from approximating the function $F$ by a functional Monte Carlo estimate defined as
\[
   F_n (\theta ) = \frac{1}{n} \sum_{i=1}^n f(\theta, X_i) ,
\]
where $X_1,\ldots,X_n$ is a random sample from $P$. Then the minimizer of the first stochastic optimization problem is approximated by the minimizer $\theta_n$ of the functional estimate $F_n$. Following the reference textbook \citet{shapiro2014lectures}, the analysis of $\theta_n$ is carried out through error bounds for the Monte Carlo estimate $F_n(\theta)$ uniformly in $\theta \in \Theta$.
When sampling from $P$ is expensive, one may want to reduce the variance of the Monte Carlo estimate $F_n(\theta)$ by the method of control variates or some other method. In that case, a control on the error uniformly in $\theta$ is required.

\paragraph{Quantile estimation in simulation modeling.}

Quantiles are of prime importance when it comes to measure the uncertainty of random models \citep{law2000}. When the stochastic experiments are costly, variance reduction techniques such as the use of control variates are helpful \citep{hesterberg1998control,cannamela2008controlled}.
Let $F(y) = \pr \{ g(X) \leq y \}$, for $y\in \reals$, be the cumulative distribution function of a transformation $g : \Xc \to \reals$ of a random element $X$ in some space $\Xc$.
Suppose the interest is in the quantile $F^{-}(u) = \inf\{y \in \reals\,:\, F(y)\geq u\}$ for $u\in (0,1)$.
The functions $f_y : \Xc \to \reals$ to be integrated with respect to the distribution $P$ of $X$ are thus the indicators $x \mapsto f_y(x) = I\{g(x) \le y\}$, indexed by $y \in \reals$.
It is necessary to control the accuracy of an estimate of the probability $F(y) = \expec[f_y(X)]$ uniformly in $y \in \reals$ in order to have a control on the accuracy of an estimate of the quantile $F^{-}(u)$, even for a single $u \in (0, 1)$; see for instance Lemma~12 in \cite{portier2018weak}.
If drawing samples from $P$ or evaluating $g$ is expensive, it may be of interest to limit the number of Monte Carlo draws $X_i$ and function evaluations $g(X_i)$.
%The analysis of the quantile estimate can be executed relying on a uniform bound on the error of the cumulative distribution function estimate---see for instance Lemma~12 in \cite{portier2018weak}.
Finally, note that due to the formulation of a quantile as the minimiser of the expectation of the check function, see e.g., \citet{hjort2011asymptotics}, this example is an instance of the stochastic programming framework described before.

\paragraph{Related results.}
The previous examples underline the need of error bounds for Monte Carlo methods that are uniform over a family of response functions. The uniform consistency of standard Monte Carlo estimates over certain collections of functions can easily be shown by relying on Glivenko-Cantelli classes \citep{vdvaart+w:1996}; see for instance \citet{shapiro2014lectures} for applications to stochastic programming problems. Similarly, uniform convergence rates for standard Monte Carlo estimates can be derived from classical empirical process theory. We refer to \cite{gine+g:02} and the references therein for suprema over VC-type classes and to \cite{kloeckner2020empirical} and the references therein for suprema over H\"older-type classes. 
	%(Note that for VC classes, the rate of convergence, typically in $1/\sqrt n $, is different from the one for H\"older classes, which is only $n^{-s/d}$ when $d/2 > s$.) \js{Il faudra introduire $s$.}
For variance reduction methods based on adaptive importance sampling, uniform consistency has been proven recently in \citet{delyon+p:2018} and \citet{feng2018uniform}. For control variates, however, we are not aware of any uniform error bounds and we believe the next results to be the first of their kind.

\subsection{Mathematical background for control variates}
\label{sec:CVback}

Let $\Fc \subset L^2(P)$ be a collection of square-integrable, real-valued functions $f$ on a probability space $(\mathcal{X}, \mathcal{A}, P)$ of which we would like to calculate the integral $P(f) = \int_{\mathcal{X}} f(x) \, \diff P(x)$. Let $X_1,\ldots,X_n $ be  independent random variables  taking values in $\mathcal{X}$ and with common distribution $P$.
The standard Monte Carlo estimate of $P(f)$ simply takes the form $P_n(f)=\frac{1}{n}\sum_{i=1}^n f(X_i)$. However, this estimator may converge slowly to $P(f)$ due to a high variance. To tackle this issue, it is common practice to use control variates, which are functions in $L^2(P)$ with known integrals. Without loss of generality, we can center the control variates $g_{n,1}\ldots,g_{n,d_n}$ and assume they have zero expectation, that is, $P(g_{n,k})=0$ for all $k\in \{1,\ldots,d_n\}$.
Let $g_n = (g_{n,1},\ldots,g_{n,d_n}) $ denote the $\mathbb R^{d_n}$-valued function with the $d_n$ control variates as elements and put $h_n = (1,g_n^\T)^\T$. Similarly as before, we assume that the Gram matrix $P(g_n g_n^\T)$ is invertible.  The control variate Monte Carlo estimate of $P(f)$ is given by $\hat \alpha_{n,f}$ defined as \citep[see for instance][Section~1]{portier+s:2019}, 
\begin{equation}
\label{eq:alphanf}
(\hat \alpha_{n,f}, \hat \beta_{n,f}) \in \argmin _ {\alpha \in \mathbb R, \, \beta \in \mathbb R^{d_n}}  P_n ( f   - \alpha -  g_n  ^\T  \beta)^2.
\end{equation}
The vector $\hat \beta_{n,f}$ contains the regression coefficients for the prediction of $f$ based on the covariates $h_n$. Remark that the control variate integral estimate $\hat \alpha_{n,f}$ coincides with the integral of the least square estimate of $f$, i.e., $\hat \alpha_{n,f} = P(\hat \alpha_{n,f}+ g_n  ^\T \hat \beta_{n,f})$. In addition, since $\hat \alpha_{n,f}$ can be expressed as a weighted estimate $\sum_{i=1}^n w_{i} f(X_i)$ where the weights $(w_i)_{i=1,\ldots, n}$ do not depend on the integrand $f$, there is a computational benefit to integrating multiple functions \citep[Remark 4]{leluc2019control}. It is useful to define
\begin{equation*}
(\alpha_{n,f},  \beta_{n,f}) \in \argmin _ {\alpha \in \mathbb R, \, \beta \in \mathbb R^d}  P  ( f   - \alpha -  g_n   ^\T  \beta)^2,
\end{equation*}
as well as the residual function 
\[
	\eps_{n,f}  = f - \alpha_{n,f} - g_n^\T  \beta_{n,f}.
\]
Note that $ \alpha_{n,f}  = P(f)$.
If $ \beta_{n,f} $ would be known, the resulting oracle estimator would be
\begin{equation}
\label{eq:alphanforacle}
\hat \alpha_{n,f}^{\mathrm{or}} = P_n[f-   g_n  ^\T \beta_{n,f} ].
\end{equation}
The question raised in the next section is whether the control variate estimate $\hat \alpha_{n,f}$ can achieve a similar accuracy uniformly in $f \in \Fc$ as the oracle estimator $\hat \alpha_{n,f}^{\mathrm{or}}$.

\subsection{Uniform error bounds}

Motivated by the examples above, we provide an error bound for the control variate Monte Carlo estimate 
$ \hat \alpha_{n,f}$ in~\eqref{eq:alphanf} uniformly in $f \in \Fc$. Before doing so, we give a uniform error bound for the oracle estimate $\hat \alpha_{n,f}^{\mathrm{or}}$ in~\eqref{eq:alphanforacle}. This serves two purposes: first, it will be useful in the analysis of $ \hat \alpha_{n,f}$ and second, it will provide sufficient conditions for the two estimates to achieve the same level of performance.
Recall $M_n$, $\gamma_n^2$ and $L_n^2$ in \eqref{eq:Mn} and \eqref{eq:Ln} and put
\[
	\sigma_n^2 = \supf P(\eps_{n,f}^2).
\]
A new assumption, $M_n^2 = \Oh \left(\sigma_n ^2  n / \log (n ) \right)$ as $n \to \infty$, in the same vein but weaker\footnote{Since $L_n^2 = M_n^2 \ninf{q_n}$ and $\gamma_n^2 \le \sigma_n^2 \ninf{q_n}$.} than $ L_n^2 = \Oh \left(\gamma_n ^2 \sqrt{ n / \log (n ) }\right)$, turns out to be useful to obtain the result.

\begin{theorem}[Uniform bound on error of oracle estimator]
	\label{prop:oracle_bound}
	Assume the framework of Section~\ref{sec:CVback} and suppose that Conditions \ref{cond:qn}, \ref{cond:F} and \ref{cond:VC} hold. If $\liminf _ {n\to \infty}  n^{\alpha} M_n> 0 $ for some $\alpha >0$ and if $ M_n^2 = \Oh \left(\sigma_n ^2  n / \log (n ) \right) $, then
\begin{align*}
\sup_{f\in \mathcal F}  \left|\hat \alpha_{n,f}^{\mathrm{or}} - P(f) \right| 
 &= \Op\left(\sigma_n \sqrt{  n^{-1} \log( n )} \right), \qquad n \to \infty.
\end{align*}
\end{theorem}

The proof of Theorem \ref{prop:oracle_bound} is provided in Appendix \ref{app:proofMC1}. The derivation of the stated rate relies on the property that the residual class $\Ec_n = \{ \eps_{n,f} : f \in \Fc\}$ is a VC-class of functions (as detailed in the proof of Proposition~\ref{prop:cover}). Indeed, noticing that 
$\hat \alpha_{n,f}^{\mathrm{or}} - P(f) =  P_n  (\eps_{n,f})$ allows to rely on the next proposition, dedicated to suprema of empirical process.

\begin{proposition}[Bound of supremum of empirical process]
	\label{prop:general term_Pn f}
	On the probability space $(\Xc, \Ac, P)$, let $\Sc$ be a VC-class of parameters $(w, B)$ with respect to the constant envelope $U \ge \sup_{\scfunc \in \Sc}\ninf{\scfunc}$. Suppose the following two conditions hold:
	\begin{compactenum}[(i)]
		\item%\label{ass:i} 
		$\tau^2 \ge \sup_{\scfunc\in\Sc} \Var_P(\scfunc)$ and $\tau \le 2U$;
		\item%\label{ass:ii} 
		$w \ge 1$ and $B \ge 1 $. %$3\sqrt{\rme}$.
	\end{compactenum}
	Then, for $P_n$ the empirical distribution of an independent random sample $X_1, \ldots, X_n$ from $P$, we have with probability $1-\delta$: 
	\begin{align*}
		\sup_{\scfunc \in\Sc} \left| P_n(\scfunc) - P(\scfunc) \right|
		&\leq L \left(\tau \sqrt{w n^{-1} \log(L \theta/\delta)} + U w n^{-1} \log(L \theta/\delta ) \right),
	\end{align*}	
	with $\theta = B U / \tau $ and $L>0$ a universal constant.
\end{proposition}

The proof of Proposition~\ref{prop:general term_Pn f} is given in Appendix~\ref{supp:prop:general term_Pn f}. In the proof, we bound the expectation of the supremum by combining a well-known symmetrization inequality \citep[Lemma~2.3.1]{vdvaart+w:1996} with Proposition~2.1 in \citet{gine2001consistency}, and we will find a rate on the deviation of the supremum around its expectation by Theorem~1.4 in \citet{talagrand1996new}. Compared to existing results such as Proposition 2.2 stated in \cite{gine2001consistency}, our version is more precise due to the explicit role played by the VC constants in the bound.

The next result follows from an application of Corollary \ref{cor:simple-rate} and Theorem \ref{prop:oracle_bound} combined with some other bounds that are standard when analyzing control variates estimates.

\begin{theorem}[Uniform error bound on control variate Monte Carlo estimator]
	\label{th:unif_cv}
Assume the framework of Section~\ref{sec:CVback} and suppose that Conditions \ref{cond:qn}, \ref{cond:F} and \ref{cond:VC} hold. If $\liminf _ {n\to \infty}  n^{\alpha} M_n> 0 $ for some $\alpha >0$ and if $ L_n^2 = \Oh \left( \gamma_n ^2 \sqrt{ n / \log (n ) } \right) $, then
\begin{equation*}
\supf  \left|\hat \alpha _ f - P(f) \right| 
= \Op\left(   \sigma_n \sqrt{  n^{-1} \log( n )} \left( 1  + \sqrt {  d_n n^{-1} \ninf {q_n}       }   \right)  \right),
\qquad n \to \infty.
\end{equation*}
\end{theorem}

The proof of Theorem \ref{th:unif_cv} is provided in Appendix \ref{app:proofMC2}.  Compared to the error bound given in Theorem~\ref{prop:oracle_bound} for the oracle estimator, the error bound in Theorem~\ref{th:unif_cv} for the control variate estimator has an additional term. This term, which is due to the additional learning step that is needed to estimate the optimal control variate, vanishes as soon as $d_n \ninf{q_n} = \oh(n)$ as $n \to \infty$. This condition, which was used in \cite{newey1997convergence} as well as in \cite{portier+s:2019}, is meaningful as it relates the model complexity to the sample size, i.e., the computing time of the experiment.

\appendix

\section{Auxiliary lemmas}
\label{supp:lemmas}

% !TEX root = risk-bounds.tex

\begin{proof}[Proof of Lemma~\ref{lem:Gninv}]
The lemma states an asymptotic lower bound for the smallest eigenvalue of the empirical Gram matrix $P_n(\hbar_n \hbar_n^\T)$ under two alternative conditions, (a) or (b).

First, suppose first condition~(a) holds.
Lemma~3 in \citet{portier+s:2019} states that $P_n(h_n h_n^\T)$ and thus $P_n(\hbar_n \hbar_n^\T)$ fails to be invertible with probability at most $n^{-1} P(q_n^2)$. This probability tends to zero by assumption.
	
Recall the spectral norm $|\cdot|_2$ and let $\abs{A}_F = (\sum_{i,j} A_{ij}^2)^{\half}$ denote the Frobenius norm of a matrix $A$. Lemma~2 in \citet{portier+s:2019} states that $\expec[ \abs{P_n(\hbar_n \hbar_n^\T) - I_{d_n}}_F^2]$ is bounded by $n^{-1} P(q_n^2)$ and thus converges to zero as $n \to \infty$. But then the same is true for $\expec[ \abs{P_n(\hbar_n \hbar_n^\T) - I_{d_n}}_2^2 ]$, since $\abs{A}_2 \le \abs{A}_F$ for any square matrix $A$. It follows that $\abs{P_n(\hbar_n \hbar_n^\T) - I_{d_n}}_2 = \op(1)$ as $n \to \infty$.

On the event that $P_n(\hbar_n \hbar_n^\T)$ is invertible, we have
\begin{align*}
	\abs{P_n(\hbar_n \hbar_n^\T)^{-1}}_2
	&=
	\abs{I_{d_n} + P_n(\hbar_n \hbar_n^\T)^{-1} \{I_{d_n} - P_n(\hbar_n \hbar_n^\T)\}}_2 \\
	&\le
	1 + \abs{P_n(\hbar_n \hbar_n^\T)^{-1}}_2 \cdot \abs{P_n(\hbar_n \hbar_n^\T) - I_{d_n}}_2
\end{align*}
from which
\[
	\frac{1}{\lambda_{\min}\{P_n(\hbar_n \hbar_n^\T)\}}
	=
	\abs{P_n(\hbar_n \hbar_n^\T)^{-1}}_2
	\le \frac{1}{1 - \abs{P_n(\hbar_n \hbar_n^\T) - I_{d_n}}_2}
	= 1 + \op(1), \qquad n \to \infty.
\]

Second, suppose condition~(b) holds.
Lemma~A.2 in \citet{leluc2019control}, which is based on Theorem~5.1.1 in \citet{tropp2015introduction}, states that for $0 < \delta < 1$ and for $n$ sufficiently large such that $n > 2 \ninf{q_n} \log(d_n / \delta)$, we have
\[
	\pr\left[ \lambda_{\min}\{P_n(\hbar_n \hbar_n^\T)\} \le 1 - \sqrt{(2/n) \ninf{q_n} \log(d_n / \delta)} \right]
	\le \delta. 
\]
By assumption, $\ninf{q_n} \log(d_n / \delta) = \oh(n)$ as $n \to \infty$, for any $0 < \delta < 1$. It follows that $\lambda_{\min}\{P_n(\hbar_n \hbar_n^\T)\} \ge 1 - \op(1)$ as $n \to \infty$.
\end{proof}

\begin{lemma}\label{lemma:maj M_n}
	If Conditions~\ref{cond:qn} and~\ref{cond:F} hold, then
	\[
	\supf \ninf{\eps_{n,f}} 
	\le \ninf{F} + \br{\ninf{q_n} P(F^2)}^{\half}
	\le \prr{1 + \ninf{q_n}^{\half}} \ninf{F}.
	\]
\end{lemma}

\begin{proof}[Proof of Lemma~\ref{lemma:maj M_n}]
	Let $f \in \Fc$. We have $f = h_n^\T \beta_{n,f} + \eps_{n,f}$ with $\beta_{n,f} = P(h_n h_n^\T)^{-1} P(h_n f)$ and $P(h_n \eps_{n,f}) = 0$. Since $\hbar_n = P(h_n h_n^\T)^{-\half} h_n$, we get $f = \hbar_n^\T P(\hbar_n f) + \eps_{n,f}$. Now $\hbar_n$ and $\eps_{n,f}$ are orthogonal while $P(\hbar_n \hbar_n^\T) = I_{d_n}$, so that
	\[ 
		P(f^2) 
		= \abs{P(\hbar_n f)}_2^2 + P(\eps_{n,f}^2) 
		\ge \abs{P(\hbar_n f)}_2^2.
	\]
	It follows that
	\[
		[\hbar_n^\T P(\hbar_n f)]^2
		=
		q_n \abs{P(\hbar_n f)}_2^2
		\le
		q_n P(f^2)
		\le
		q_n P(F^2).
	\] 
	But then
	\[
		\abs{\eps_{n,f}} 
		\le \abs{f} + \abs{\hbar_n^\T P(\hbar_n f)} 
		\le \abs{F} + [q_n P(F^2)]^{\half}.
	\]
	Since $P(F^2) \le \ninf{F}^2$, the result follows.
\end{proof}

\section{Proof of Proposition~\ref{prop:cover}}
\label{supp:prop:cover}

% !TEX root = risk-bounds.tex

The idea of the proof is to create a grid of functions on $\Xc$ based on a covering of $\Fc$ to cover $\Ec_n = \{ \eps_{n,f} : f \in \Fc\}$. From this grid, we will deduce coverings of $\Gc_n$ and $\Gc_n^{(d)}$.

\paragraph{Step 1: covering of $\Ec_n$.}

By assumption, the class $\Fc$ is VC of parameters $(v,A)$ with respect to an envelope $F$. That means that for any $0 < \eta < 1$ and for any probability measure $Q$ on $\Xc$, we have
\[
	\Nc\left(\Fc, L^2(Q), \eta \|F\|_{L^2(Q)}\right) \le \left( \frac{A}{\eta} \right)^{v} .
\]
Moreover, a single ball centered at the constant function equal to zero and with radius $\| F \|_{L^2(Q)}$ is enough to cover $\Fc$. Thus, for any $\eta \in (0, \infty)$, the covering number is bounded from above by
\[	
	\Nc\left(\Fc, L^2(Q), \eta\right) \le \left( \frac{A \|F\|_{L^2(Q)}}{\eta} \right)^{v} \vee 1.
\]
Fix $\eta > 0$, write $\eta_P = \eta/(4 \ninf{q_n}^{\half})$ and $\eta_Q = \eta/4$, and define the covering numbers
\begin{align}
\label{eq:NPNQ}
	N_P &= \Nc\left(\Fc, L^2(P), \eta_P / 2 \right), & 
	N_Q &= \Nc\left(\Fc, L^2(Q), \eta_Q / 2 \right)
\end{align}
associated to the open balls
\begin{align*}
	B_P(f, \delta) &= \left\{g \in L^2(P) : \|g-f\|_{L^2(P)} < \delta \right\}, & 
	B_Q(f, \delta) &= \left\{g \in L^2(Q) : \|g-f\|_{L^2(Q)} < \delta \right\},
\end{align*}
for $\delta > 0$. The balls in the definition of the covering numbers in \eqref{eq:NPNQ} have their centers in $L^2(P)$ and $L^2(Q)$ but not necessarily in $\Fc$. At the price of doubling the radii, the triangle inequality permits us to find functions $f_1^{(P)},\ldots, f_{N_P}^{(P)}$ and $f_1^{(Q)},\ldots, f_{N_Q}^{(Q)}$ in $\Fc$ such that
\begin{align*}
\label{eq:FcBPBQ}
\Fc &\subset \bigcup_{i=1}^{N_P} B_P( f_i^{(P)}, \eta_P ), &
\Fc &\subset \bigcup_{j=1}^{N_Q} B_Q( f_j^{(Q)}, \eta_Q ).
\end{align*}
Therefore, the class $\Fc$ is covered by the union of the intersections between the balls, that is to say
\begin{equation}\label{eq:inclusion}
	\Fc \subset \bigcup_{\substack{1 \le i\le N_P \\ 1\le j\le N_Q}} 
	\left( B_P( f_i^{(P)}, \eta_P ) \cap B_Q( f_j^{(Q)}, \eta_Q ) \right).
\end{equation}
Define the support of this covering as
\[
\mathsf{S} = \left\{(i,j) \in  \left\{1,\ldots, N_P\right\} \times \left\{1,\ldots, N_Q\right\}: B_P( f_i^{(P)}, \eta_P ) \cap B_Q( f_j^{(Q)}, \eta_Q ) \neq \emptyset \right\}.
\]
For every $(i,j)\in\mathsf{S}$, we fix an arbitrary function $f_{i,j}\in B_P( f_i^{(P)}, \eta_P ) \cap B_Q( f_j^{(Q)}, \eta_Q )$. 

Let $f\in \Fc$ and let $(i,j) \in \mathsf{S}$ be such that $f \in B_P( f_i^{(P)}, \eta_P ) \cap B_Q( f_j^{(Q)}, \eta_Q )$. We will show that $\norm{\eps_{n,f} - \eps_{n,f_{i,j}}}_{L^2(Q)} \le \eta$. Since $f$ and $f_{i,j}$ belong to the same intersection in \eqref{eq:inclusion}, we have
\begin{align}
\label{eq:ffijL2PL2Q}
	\norm{f - f_{i,j}}_{L^2(P)} &< 2 \eta_P, &
	\norm{f - f_{i,j}}_{L^2(Q)} &< 2 \eta_Q.	
\end{align}
The residual functions can be expressed in terms of the whitened feature map $\hbar_n$ via
\begin{align*}
	\eps_{n,f} &= f - \hbar_n^\T P(\hbar_n f), &
	\eps_{n,f_{i,j}} &= f_{i,j} - \hbar_n^\T P(\hbar_n f_{i,j}).
\end{align*}	
By the triangle inequality, we find
\begin{align} 
	\label{eq:residuals}
	\norm{ \eps_{n,f} - \eps_{n, f_{i,j}} }_{L^2(Q)} 
	\le \norm{ f - f_{i,j} }_{L^2(Q)} + \norm{ \hbar_n^\T P[\hbar_n(f - f_{i,j})] }_{L^2(Q)}.
\end{align}
Recall $M_n \ge \sup_{f \in \Fc} \ninf{\eps_{n,f}}$, a constant envelope for the class $\Ec_n$, and recall the leverage function $q_n = |\hbar_n|_2^2$. The Cauchy--Schwarz inequality and the orthonormality of $\prn{\hbar_{n,1},\ldots,\hbar_{n,d_n}}$ give
\begin{align}
\nonumber
\norm{ \hbar_n^\T P[\hbar_n(f - f_{i,j})] }_{L^2(Q)}^2
&= \int_{y\in\Xc} \left\{ \hbar_n(y)^\T \int_{x\in \Xc} \hbar_n(x) (f-f_{i,j})(x) \, \diff P (x) \right\}^2 \,\diff Q(y) \\
\nonumber
&\le \int_{y\in\Xc} |\hbar_n(y)|_2^2 \left| \int_{x\in \Xc} \hbar_n(x) (f-f_{i,j}) (x) \, \diff P (x) \right|_2^2 \,\diff Q(y) \\
\label{eq:hnPhnfL2Q}
&\le \ninf{q_n} \norm{f - f_{i,j}}^2_{L^2(P)}.
\end{align}
The combination of \eqref{eq:ffijL2PL2Q}, \eqref{eq:residuals} and \eqref{eq:hnPhnfL2Q} yields
\[
	\norm{\eps_{n,f} - \eps_{n, f_{i,j}}}_{L^2(Q)} 
	< 2 \eta_Q + 2 \ninf{q_n}^{\half} \eta_P
	= \eta/2 + \eta/2
	= \eta.
\]
We have thus constructed the covering
\begin{equation*}
\Ec_n \subset \bigcup_{(i,j)\in\mathsf{S}} B_Q( \eps_{n,f_{i,j}}, \eta )
\end{equation*}
of $\Ec_n$ with $L^2(Q)$ balls of radius at most $\eta$. The covering number of $\Ec_n$ is thus bounded by
\[
\Nc\left(\Ec_n, L^2(Q), \eta\right) 
\le \# \mathsf{S} 
\le \Nc\left(\Fc, L^2(P), \eta_P / 2 \right) \cdot \Nc\left(\Fc, L^2(Q), \eta_Q / 2 \right).
\]
Using the definition of a VC-class and since $\norm{q_n}_\infty \ge P(q_n) =  P(|\hbar_n|_2^2) = d_n \ge 1$, a case-by-case analysis reveals that
\begin{align}
\nonumber
	\Nc\left(\Ec_n, L^2(Q), \eta\right)
	&\le 
	\left(\frac{2 A \norm{F}_{L^2(P)}}{\eta_P}\vee 1 \right)^v \cdot
	\left(\frac{2 A \norm{F}_{L^2(Q)}}{\eta_Q} \vee 1 \right)^v \\
\nonumber
	&\le
	\left(\frac{64 A^2 \norm{F}_{\infty}^2\ninf{q_n}^{\half}}{\eta^2}\right)^v 
	\vee \left(\frac{8 A \norm{F}_{\infty} \ninf{q_n}^{\half}}{\eta}\right)^v \vee 1 \\
%	&= \left(\frac{64 A^2 \norm{F}_{\infty}^2\ninf{q_n}^{\half}}{\eta^2}\right)^v \1_{\eta\le t_-}
%	+ \left(\frac{8 A \norm{F}_{\infty} \ninf{q_n}^{\half}}{\eta}\right)^v \1_{t_- < \eta < t_+}
%	+ \1_{t_+ \le \eta} \\
\label{eq:NE}
	&\le \left(\frac{8 A \norm{F}_{\infty} \ninf{q_n}^{\half}}{\eta}\right)^{2v} \vee 1.
%	 \\
%\nonumber
%	&= \left(\frac{A_n M_n}{\eta}\right)^{2v} \vee 1,
\end{align}

Lemma~\ref{lemma:maj M_n} implies that the sequence $A_n$ defined in~\eqref{eq:def A_n} satisfies $A_n\ge 1$.
%, where we write
%\begin{equation}
%\label{eq:An8}
%	A_n = 8 A \ninf{F} \ninf{q_n}^{\half} / {M_n}.
%\end{equation}
From~\eqref{eq:NE} we deduce
\[
	\forall \eta\in(0,1],\qquad
	\Nc\prr{\Ec_{n},L^2(Q),\eta M_n} \le (A_n/\eta)^{2v}.
\]
Therefore, the residual class $\Ec_n$ is VC of parameters $(2v, A_n)$ with respect to the envelope $M_n$.

\paragraph{Step 2: covering of $\Gc_n$.}

Consider two functions $f, \tilde{f}\in\Fc$ and a probability measure $Q$ on $\Xc^2$ with marginals $Q_1, Q_2$ on $\Xc$, that is, $Q_1(B) = Q(B \times \Xc)$ and $Q_2(B) = Q(\Xc \times B)$ for measurable $B \subset \Xc$. By definition, every function in $\Gc_n$ is written as $(x,y)\mapsto \hbar_n(x)^\T \hbar_n(y) \eps_{n,f}(x)\eps_{n,f}(y)$. The Cauchy--Schwarz inequality gives 
\begin{equation}
\label{eq:hxhybound}
	\forall (x, y) \in \Xc^2, \qquad
	|\hbar_n(x)^\T \hbar_n(y)|^2 
	\le |\hbar_n(x)|_2^2 \, |\hbar_n(y)|_2^2 
	= q_n(x) q_n(y) 
	\le \ninf{q_n}^2.
\end{equation}

For $f, \tilde{f} \in \Fc$ and $(x, y) \in \Xc^2$, define 
\[
	g_{n,f,\tilde{f}}(x, y) 
	= \eps_{n,f}(x) \hbar_n(x)^\T \hbar_n(y) \eps_{n,\tilde{f}}(y).
\]
Note that $g_{n,f} = g_{n,f,\tilde{f}}$. By the Minkowski inequality,
\[
	\norm{ g_{n, f} - g_{n, \tilde{f}} }_{L^2(Q)}
	\le
	\norm{ g_{n, f} - g_{n, f, \tilde{f}} }_{L^2(Q)}
	+
	\norm{ g_{n, f, \tilde{f}} - g_{n, \tilde{f}} }_{L^2(Q)}.
\]
Let us look at square of the first term on the right-hand side: by~\eqref{eq:hxhybound},
\begin{align*}
	\norm{ g_{n, f} - g_{n, f, \tilde{f}} }_{L^2(Q)}^2
	&= \int_{\Xc^2} 
		\eps_{n,f}(x)^2 
		\left|\hbar_n(x)^\T \hbar_n(y)\right|^2 
		\left( \eps_{n,f}(y) - \eps_{n,\tilde{f}}(y) \right)^2
		\diff Q(x, y) \\
	&\le \int_{\Xc^2} 
		\eps_{n,f}(x)^2 q_n(x) q_n(y)
		\left( \eps_{n,f}(y) - \eps_{n,\tilde{f}}(y) \right)^2
		\diff Q(x, y) \\
	&\le \ninf{q_n \eps_{n,f}^2}
	\int_{\Xc} 
		q_n(y) 
		\left( \eps_{n,f}(y) - \eps_{n,\tilde{f}}(y) \right)^2
	\diff Q_2(y) \\
	&\le \ninf{q_n \eps_{n,f}^2} \ninf{q_n} \norm{\eps_{n,f} - \eps_{n,\tilde{f}}}_{L^2(Q_2)}^2.
\end{align*}
The term $\norm{g_{n,f,\tilde{f}} - g_{n,\tilde{f}}}_{L^2(Q)}$ can be treated similarly, yielding
\begin{align}
	\lefteqn{
	\norm{ g_{n, f} - g_{n, \tilde{f}} }_{L^2(Q)}
	} \nonumber \\
	\nonumber
	&\le 
	\ninf{q_n \eps_{n,f}^2}^{1/2} \ninf{q_n}^{1/2} \norm{\eps_{n,f} - \eps_{n,\tilde{f}}}_{L^2(Q_2)}
	+
	\ninf{q_n \eps_{n,\tilde{f}}^2}^{1/2} \ninf{q_n}^{1/2} \norm{\eps_{n,f} - \eps_{n,\tilde{f}}}_{L^2(Q_1)} \\
	\label{eq:gnfft}
	&\le M_n \ninf{q_n}
	\left(
	\norm{\eps_{n,f} - \eps_{n,\tilde{f}}}_{L^2(Q_1)}
	+
	\norm{\eps_{n,f} - \eps_{n,\tilde{f}}}_{L^2(Q_2)}
	\right).
\end{align}

Fix $\eta > 0$. Following the approach in Step~1, we can for $\ell \in \{1, 2\}$ construct a covering of $\Ec_n$ by $L^2(Q_\ell)$-balls of at most radius $\eta$. The centers of the balls are of the form 
\[
	\forall \ell = 1, 2, \;
	\forall k = 1, \ldots, m_\ell, \qquad 
	\eps_{n,f^{(\ell)}_{k}} \in \Ec_n,
\] 
where $f^{(\ell)}_{k}$ belongs to $\Fc$ and where $m_\ell$ is the number of such balls needed, a number which is bounded by $(A_n M_n / \eta )^{2v}\vee 1$ for $A_n$ defined in~\eqref{eq:def A_n} Consider the intersections
\[
	\forall i = 1, \ldots, m_1, \; \forall j = 1, \ldots, m_2, \qquad
	B(i, j) = B_{Q_1} \bigl(\eps_{n,f^{(1)}_{i}}, \eta \bigr) \cap B_{Q_2}\bigl(\eps_{n,f^{(2)}_{j}}, \eta\bigr).
\]

The set $\Ec_n$ is covered by the union of all those intersections $B(i,j)$. For each $(i, j) \in \{1,\ldots,m_1\} \times \{1,\ldots,m_2\}$ such that $\Ec_n$ intersects $B(i,j)$, pick an arbitrary $f_{i,j} \in \Fc$ such that $\eps_{n,f_{i,j}} \in \Ec_n \cap B(i,j)$. Note that the functions $f_{i,j}$ are different from the ones denoted in the same way in Step~1.

Let $f \in \Fc$ and let $(i,j)$ be such that $\eps_{n,f}$ and $\eps_{n,f_{i,j}}$ belong to the same intersection $B(i, j)$. Since the diameters of the two balls in the definition of $B(i,j)$ are bounded by $2\eta$ in view of the triangle inequality, we find
\[
	\forall \ell = 1, 2, \qquad 
	\norm{\eps_{n,f} - \eps_{n,f_{i,j}}}_{L^2(Q_\ell)}
	< 2\eta.
\]
By \eqref{eq:gnfft}, it follows that
\[
	\norm{g_{n,f} - g_{n,f_{i,j}}}_{L^2(Q)}
	< M_n \ninf{q_n} [(2\eta) + (2\eta)]
	= 4 M_n \ninf{q_n} \eta.
\]

We find that $\Gc_n$ is covered by the union of the balls $B_Q(g_{n,f_{i,j}}, 4 M_n \ninf{q_n} \eta)$. Its covering number is thus bounded by
\begin{equation*}
%\label{eq:NGNE}
	\Nc\left( \Gc_n, L^2(Q), 4 M_n \ninf{q_n} \eta \right)
	\le m_1 m_2 \\
%\nonumber
	\le \left( \frac{A_n M_n}{\eta} \right)^{4v} \vee 1.
\end{equation*}
Rescaling $M_n \eta' = 4 \eta$, we get
\begin{equation}
\label{eq:NGM2q}
	\Nc\left( \Gc_n, L^2(Q), M_n^2 \ninf{q_n} \eta' \right)
	% \le \left( \frac{2 A_n}{\eta^\prime} \right)^{4v} \1_{\eta^\prime < 4}
	% + \1_{\eta^\prime \ge 4}
	\le \left( \frac{4 A_n}{\eta'} \right)^{4v} \vee 1.
\end{equation}
In view of \eqref{eq:hxhybound}, the functions in $\Gc_n$ are uniformly bounded by $M_n^2 \ninf{q_n}$. Since $Q$ was an arbitrary probability measure on $\Xc^2$, we conclude that $\Gc_n$ is a VC-class with parameters $(4v, 4A_n)$ with respect to the constant envelope $M_n^2 \ninf{q_n}$.

\paragraph{Step 3: covering of $\Gc_n^{(d)}$.}

Let $Q$ be a probability measure on $\Xc$ and let $\eta > 0$. In Step~1, we found functions $f_{i,j} \in \Fc$ such that $\Ec_n$ is covered by the balls $B_Q(\eps_{n,f_{i,j}}, \eta)$, and we needed at most $(A_n M_n / \eta)^{2v} \vee 1$ of such functions. For $f \in \Fc$ we can thus find $(i, j)$ such that
\[
	\norm{\eps_{n,f} - \eps_{n,f_{i,j}}}_{L^2(Q)} < \eta
\]
and thus 
\begin{align*}
	\norm{ \sqrt {q_n} \eps_{n,f} - \sqrt{q_n} \eps_{n,f_{i,j}} }_{L^2(Q)}^{2}
	%&= \textstyle\int_{x\in \Xc} \acn{ q_n(x) \eps_{n,f}^2(x) - q_n(x) \eps_{n,f_{i,j}}^2(x) } ^2 \,\diff Q(x) \\
	&\le  \ninf{q_n}  \left\| \eps_{n,f} - \eps_{n,f_{i,j}} \right\|^2 _{L^2(Q)} \\
	&< \ninf{q_n}  \eta^2.
\end{align*}
It follows that the number of $L^2(Q)$ balls of radius $  \ninf{q_n} ^{1/2} \eta$ needed to cover $\Gc_n^{(d)}$ is bounded by the number of functions $f_{i,j}$ in the construction in Step~1, and so
\begin{align*}
%\label{eq:NGdNE}
	\Nc\left( \Gc_n^{(d)}, L^2(Q),    \ninf{q_n} ^{1/2} \eta \right)
	&\le 
	\Nc\left( \Ec_n, L^2(Q), \eta \right) \\
%\nonumber
	&\le
	\left( \frac{A_n M_n}{\eta} \right)^{2v} \vee 1.
\end{align*}
Upon rescaling $M_n \eta' =  \eta$, we find
\begin{equation}
\label{eq:NGdM2q}
\Nc\left( \Gc_n^{(d)}, L^2(Q), M_n     \ninf{q_n} ^{1/2}  \eta' \right)
%&\le \left( \frac{2 A_n}{\eta^\prime} \right)^{2v} \1_{\eta^\prime < 2}
%+ \1_{\eta^\prime\ge 2} \\
\le \left( \frac{ A_n}{\eta'} \right)^{2v}\vee 1.
\end{equation}
We conclude that $\Gc_n^{(d)}$ is a VC-class with parameters $(2v, A_n)$ with respect to the constant envelope $M_n  \ninf{q_n}^{1/2}$. The proof of Proposition~\ref{prop:cover} is complete. \qed

\section{Proof of Theorem~\ref{thm:main} and Corollary~\ref{cor:simple-rate}}
\label{supp:thm:main}

% !TEX root = risk-bounds.tex

We recall the following quantities:
\begin{align*}
%\label{eq:qnt}
	M_n &= \supf \ninf{\eps_{n,f}}, &
	\gamma_n^2 &= \supf P(q_n \eps_{n,f}^2), &
	L_n^2 &= M_n^2 \ninf{q_n}.
\end{align*}
For all $f \in \Fc$, we have $P(\eps_{n,f}^2) \le P(f^2) \le P(F^2)$ and $\gamma_n \le L_n$.

\paragraph{Step 1: Overview.}

We follow the plan laid out in Section~\ref{sec:proof} in the paper. In view of \eqref{eq:worstLfbound} and Lemma~\ref{lem:Gninv}, we have
\[
\supf \left\{ L_f(\hat{\beta}_{n,f}) - L_f(\beta_f) \right\}
\le
\left\{1 + \Op(1)\right\} \supf \left| P_n(\hbar \eps_{n,f}) \right|_2^2.
\]
The inequality~\eqref{eq:n2Pndecomp} provides a bound on $n^2 \left| P_n( \hbar_n \eps_{n,f} ) \right|_2^2$ consisting of a sum of two terms %The first term is just $n \, P( q_n \eps_{n,f}^2 ) \le n \gamma_n^2$. 
which require a separate analysis (Steps~2 and~3). Finally, we collect the bounds to arrive at the stated rate (Step~4).

\paragraph{Step 2: First term in \eqref{eq:n2Pndecomp}.}

Recall the definition $\Gc_n^{(d)} = \{  q_n^{1/2} \eps_{n,f} : f \in \Fc \}$ introduced in \eqref{eq:Gn}. We first apply Corollary 3.4 given in \cite{talagrand1994sharper} to the class $\Gc_n^{(d)} $ normalized by its envelope $L_n$ so that the resulting class is valued in $[-1,1]$. We obtain that
\begin{equation}
\label{eq:step2:1}
\expec \left[  \sup_{g \in \Gc_n^{(d)}} \left|  \sum_{i=1}^n g^2(X_i) \right|   \right] 
\leq  n \gamma_n^2 + 8 L_n  \expec \left[  \sup_{g \in \Gc_n^{(d)}} \left|  \sum_{i=1}^n \eta_i  { g (X_i)} \right|  \right].
\end{equation}
Next, we apply Proposition 2.1 stated in \cite{gine2001consistency} to the class $ \Gc_n^{(d)}$ to  get
\begin{equation}
\label{eq:term2}
	 \expec \left[  \sup_{g \in \Gc_n^{(d)}} \left|  \sum_{i=1}^n \eta_i  { g (X_i)} \right|  \right]
	\leq C \left( 
		\tau_n \sqrt{w_n n \log(\theta_n)} 
		+ U_n w_n \log(\theta_n) 
	\right), 
\end{equation}
where $C>0$ is a universal constant and $\theta_n = B_n U_n / \tau_n$ and where the positive quantities $\tau_n, U_n, w_n, B_n$ need to be chosen to satisfy the following two conditions:
\begin{enumerate}[(i)]
\item\label{ass:sec-D:i} 
	$\tau_n^2 \ge \supf P (q_n \eps_{n,f}^2 )$, $U_n \ge \supf \| \sqrt {q_n} \eps_{n,f} \|_\infty$ and $U_n \ge  \tau_n$;
\item\label{ass:sec-D:ii} 
	$(w_n, B_n)$ are VC parameters of $\Gc_n^{(d)}$ with respect to the envelope $U_n$, and $w_n \ge 1$ and $B_n \ge 3\sqrt{\rme}$.
\end{enumerate}
We set
\begin{equation*}
	\tau_n = \ninf{q_n}^{1/2} (a_n\vee M_n) 	\qquad \text{and} \qquad
	U_n =  \tau_n,
\end{equation*}
where $a_n$ is a sequence tending to zero that is introduced for some technical reason in Step~3, Eq.~\eqref{eq:an}. Because $\tau_n  \geq   \ninf{q_n} ^{1/2} M_n $, Condition~\eqref{ass:sec-D:i} is satisfied. To meet~\eqref{ass:sec-D:ii}, we set $w_n = (2v) \vee 1$, independently of $n$, and 
\begin{equation}
\label{eq:Bn}
	B_n = \frac{10 A \ninf{F} \ninf{q_n}^{1/2}}{M_n \vee a_n}.
\end{equation}
Since $A \ge 1$, since $M_n \le 2 \ninf{F} \ninf{q_n}^{1/2}$ (Lemma~\ref{lemma:maj M_n}), since $\ninf{q_n} = d_n \ge 1$ and since $a_n = \oh(1)$, it follows that $B_n \ge 5 \ge 3 \sqrt{\rme}$ for all sufficiently large $n$. We get,
using the bound on the covering numbers of $\Gc_n^{(d)}$ in~\eqref{eq:NGdM2q} and the definition of $A_n$ in~\eqref{eq:def A_n}, after some calculations,
\[
	\Nc \left( \Gc_n^{(d)}, L^2(Q), U_n \eta \right)
	\le \left( \frac{A_n}{\left(1 \vee \left(a_n/M_n\right)\right) \eta} \right)^{2v} \vee 1
	\le \left(B_n / \eta\right)^{2v} \vee 1, \qquad 0 < \eta \le 1,
\]
from which Condition~\eqref{ass:sec-D:ii} follows. As $\theta_n = B_n U_n / \tau_n =  B_n$, \eqref{eq:term2} implies
\begin{align}
\nonumber
	 \expec \left[  \sup_{g \in \Gc_n^{(d)}} \left|  \sum_{i=1}^n \eta_i  { g (X_i)} \right|  \right]
	&\leq C\tau_n w_n  \left( 
		\sqrt{ n \log(B_n)} 
		+   \log(B_n) 
	\right), \\
\label{eq:step2:2}
	&= \Oh \left( \tau_n \sqrt{n} \log(B_n) \right), \qquad n \to \infty.
\end{align}
The bounds~\eqref{eq:step2:1} and~\eqref{eq:step2:2} combined with the Markov inequality give
\begin{align*}
  \sup_{f  \in \Fc }  n P_n( q_n \eps_{n,f}^2 ) =   \sup_{g \in \Gc_n^{(d)}} \left|  \sum_{i=1}^n g^2(X_i) \right|   = \Op
\left(  n \gamma_n^2 + L_n\tau_n \sqrt{n} \log(B_n)  \right),
\qquad n \to \infty.
\end{align*}
 As we will see in Step~4, the latter rate is of smaller order than the one for the third term in \eqref{eq:n2Pndecomp}, which is derived in Step~3.

\paragraph{Step 3: Third term in \eqref{eq:n2Pndecomp}.} 

The term $\sum_{1 \le i \ne j \le n} g_{n,f}(X_i, X_j)$ is a degenerate $U$-statistic of order two: for any $x \in \Xc$, we have
we have $\expec \left[ g_{n,f}(X, x) \right] = \expec \left[ g_{n,f}(x, X) \right] = 0$, where the random variable $X$ has distribution $P$.
We apply a special case of Theorem~2 in \cite{major2006estimate}, cited for convenience as Theorem~\ref{prop:U} below. The functions $g_{n,f}$ are uniformly bounded by
\[
	\sup_{x,y \in \Xc} |g_{n,f}(x, y)| 
	\le \supf \ninf{q_n \eps_{n,f}^2} 
	\le L_n^2.
\]
Let
\begin{equation*}
 \tilde{\tau}_n 
 = L_n \tau_n 
 = M_n \ninf{q_n} \left(M_n \vee a_n\right)
 =  L_n^2 \left(1 \vee \left(a_n/M_n\right)\right).
\end{equation*}
Scale the functions $g_{n,f}$ by $\tilde{\tau}_n$, yielding the class
\[
	\widetilde{\Gc}_n = \left\{ g_{n,f} / \tilde{\tau}_n : f \in \Fc \right\}
\]
of functions on $\Xc^2$ taking values in $[-1, 1]$.
For any $\eta \in (0, 1]$, by applying \eqref{eq:NGM2q} with $\eta' = \tilde{\tau}_n\eta/ \left( M_n^2\ninf{q_n} \right)$, we get, recalling the definition of $A_n$ in \eqref{eq:def A_n},
\begin{align}
\nonumber
	\Nc \left( \widetilde{\Gc}_n, L^2(Q), \eta \right)
	= \Nc \left( \Gc_n, L^2(Q), \tilde{\tau}_n \eta \right) 
%	&= \Nc \left( \Gc_n, L^2(Q), M_n^2 \ninf{q_n} \eta^\prime \right) \\
\nonumber
	&\le \prr{\frac{4A_nM_n^2\ninf{q_n}}{\tilde{\tau}_n \eta}}^{4v} \vee 1\\
\label{eq:cover:Gtilde}
	&\le \left(\frac{32 A \ninf{F} \ninf{q_n}^{1/2}}{\left(M_n \vee a_n\right) \eta}\right)^{4v} \vee 1.
\end{align}
Using Lemma~\ref{lemma:maj M_n}, we get $M_n \le 2 \ninf{F} \ninf{q_n}^{1/2}$ and since $A \ge 1$ we obtain $32A\ninf{F}\ninf{q_n}^{1/2}\ge M_n$.
In addition, the sequence $(a_n)_{n\in\N}$ defined in \eqref{eq:an} is taken such that, for $n$ sufficiently large, $32A\ninf{F}\ninf{q_n}^{1/2}\ge a_n$. Then, for $n$ sufficiently large \eqref{eq:cover:Gtilde} we have, for $B_n$ defined in~\eqref{eq:Bn},
\[
	\Nc \left( \widetilde{\Gc}_n, L^2(Q), \eta \right)
	\le \left( \tfrac{32}{10}  B_n / \eta \right)^{4v}, \qquad \eta \in (0, 1].
\]

In the terminology of \citet[p.~490]{major2006estimate},  $\widetilde{\Gc}_n$ is an $L^2$-dense class of functions with parameter $D_n$ and exponent $w$ defined by 
\begin{equation}
\label{eq:Dn} 
	D_n = \left( \tfrac{32}{10}  B_n \right)^{4v}
	\qquad \text{and} \qquad
	w = (4v) \vee 1. 
\end{equation}
For $w$, we take the maximum with $1$ in order to apply Theorem~\ref{prop:U} later on.

By the Cauchy--Schwarz inequality, we have, for any $(x, y) \in \Xc^2$ and any $f \in \Fc$,
\begin{align*}
	g_{n,f}(x, y)^2 
	&= \eps_{n,f}(x)^2 \eps_{n,f}(y)^2 \left| \hbar_n(x)^\T \hbar_n(y) \right|^2 \\
	&\le \eps_{n,f}(x)^2 \eps_{n,f}(y)^2 \, q_n(x) \, q_n(y).
\end{align*}
It follows that, for independent random variables $X_1$ and $X_2$ with common distribution $P$, we have
\[
	\forall f \in \Fc, \qquad 
	\left(\expec[ g_{n,f}(X_1, X_2)^2 ]\right)^{\half} 
	\le P(q_n \eps_{n,f}^2) \le \gamma_n^2.
\]
Upon rescaling, we get
\[
	\sup_{g \in \widetilde{\Gc}_n} \left(\expec[g(X_1, X_2)^2]\right)^{1/2}
	\le \gamma_n^2 / \tilde{\tau}_n.
\]
For a sequence $b_n \in (0, 1]$ to be determined shortly, put
\begin{equation}
\label{eq:nun}
	\nu_n = (\gamma_n^2 / \tilde{\tau}_n) \vee b_n.
\end{equation}
We have $\nu_n \le 1$; recall that $\tilde{\tau}_n \ge L_n^2 \ge \gamma_n^2$. Moreover, $\nu_n^2$ is an upper bound of the second moments of the functions in $\widetilde{\Gc}_n$. Theorem~\ref{prop:U} yields
\begin{equation}
\label{eq:major}
	\pr \left( 
	\supf \left| \sum_{1 \le i \ne j \le n} g_{n,f}(X_i, X_j) \right| 
	\ge 2 n \nu_n \tilde{\tau}_n y 
	\right)
	\le
	C D_n \, \rme^{-\alpha y}
\end{equation}
for all $y \in [y_{n,-}, y_{n,+}]$, where, for some universal constants $\alpha$, $C$, and $K$, the endpoints of the interval are (note from \eqref{eq:Dn} that $D_n \ge  B_n^{4v} \ge 1$ so $\log D_n \ge 0$)
\begin{equation*}
%	\label{eq:yrange}
	y_{n,-} = K \left( w + \frac{\log D_n}{\log n} \right)^{3/2} \log(2/\nu_n)
	\qquad \text{and} \qquad
	y_{n,+} = n \nu_n^2.
\end{equation*}

We still need to determine $a_n$ and $b_n$. We would like to choose $y$ in \eqref{eq:major} in such a way that the right-hand side is bounded by a pre-specified $\delta \in (0, 1]$. Therefore, we need to ensure two things:
\begin{align}
\label{eq:toensure:a}
	y_{n,-} &\le y_{n,+}, \qquad \text{and} \\
\label{eq:toensure:b}
	C D_n \rme^{-\alpha y_{n,+}} &\to 0, \qquad n \to \infty.
\end{align}
We need \eqref{eq:toensure:a} in order to ensure that the interval $[y_{n,-}, y_{n,+}]$ on which \eqref{eq:major} holds is non-empty; we need \eqref{eq:toensure:b} to ensure that for any $\delta \in (0, 1]$, we can find $y \in [y_{n,-}, y_{n,+}]$ such that the right-hand side of \eqref{eq:major} is bounded by $\delta$.
To this end, define
\begin{align}
\label{eq:Dn*}
	D_n^* 
	&= \exp \left[ \left\{ \left(\frac{2n}{K \log n}\right)^{2/3} - w \right\} \log n \right], \\
\label{eq:an}
	a_n 
	&= \frac{32 A \ninf{F} \ninf{q_n}^{1/2}}{(D_n^*)^{1/(4v)}}, \\
\label{eq:bn}
	b_n 
	&= \left( \frac{K \log n}{2n} \right)^{1/2} \left( w + \frac{\log D_n}{\log n} \right)^{3/4}.
\end{align}
To explain these definitions, note that $D_n^*$ is the solution to
\begin{equation*}
\label{eq:Dn*eq}
	\left(w + \frac{\log D_n^*}{\log n}\right)^{3/2} = \frac{2n}{K \log n},
\end{equation*}
whereas $a_n$ is chosen in such a way that
\begin{equation}
\label{eq:Dn*Dn}
	D_n = D_n^* \wedge \left(\frac{32 A \ninf{F} \ninf{q_n}^{1/2}}{M_n}\right)^{4v}.
\end{equation}
Furthermore, since $D_n \le D_n^*$, we have, as required earlier,
\[
	b_n \le \left( \frac{K \log n}{2n} \right)^{1/2} \left( w + \frac{\log D_n^*}{\log n}\right)^{3/4} = 1.
\]
We can now verify that both \eqref{eq:toensure:a} and~\eqref{eq:toensure:b} hold:
\begin{itemize}
\item The definition of $\nu_n$ in \eqref{eq:nun} implies $\nu_n \ge b_n$, from which we get the chain of inequalities
\begin{equation}
\label{eq:chainyn+-}
	y_{n,+} = n \nu_n^2 \ge n b_n^2 = \frac{K \log n}{2} \left( w + \frac{\log D_n}{\log n}\right)^{3/2}
	= \frac{\log n}{2 \log(2/\nu_n)} y_{n,-}
	\ge y_{n,-},
\end{equation}
which is \eqref{eq:toensure:a}. To see the last inequality in \eqref{eq:chainyn+-}, note that $1/\nu_n \le 1/b_n = \oh(n^{1/2})$, from which $(2/\nu_n)^2 = \oh(n)$ as $n \to \infty$ and thus $2 \log(2/\nu_n) \le \log n$ for sufficiently large $n$. 
\item 
Enlarging $K$ if necessary to ensure that $\alpha K \ge 2$, we have, in view of $y_n \ge n b_n^2$ and \eqref{eq:bn},
\begin{align*}
	\log D_n - \alpha y_{n,+}
	&\le
	\log D_n - \alpha n b_n^2 \\
	&=
	\log D_n - \frac{\alpha K}{2} (\log n) \left( w + \frac{\log D_n}{\log n}\right)^{3/2}	\\
	&\le \log D_n - (\log n) \left( 1 + \frac{\log D_n}{\log n}\right)^{3/2} \\
	&= (- \log n) \left[ \left( 1 + \frac{\log D_n}{\log n}\right)^{3/2} - \frac{\log D_n}{\log n} \right] \\
	&\le - \log n \to -\infty,
\end{align*}
from which \eqref{eq:toensure:b} follows.
\end{itemize}

Let $\delta \in (0, 1]$ and define
\[
	y_n(\delta) = y_{n,-} \vee \left(\alpha^{-1} \log( C D_n / \delta) \right).
\]
We already know from \eqref{eq:chainyn+-} that $y_{n,-} \le y_{n,+}$ for large $n$. Moreover, since $y = \alpha^{-1} \log( C D_n / \delta)$ is the solution to $C D_n \rme^{-\alpha y} = \delta$, the asymptotic relation in \eqref{eq:toensure:b} implies $\alpha^{-1} \log (C D_n / \delta) \le y_{n,+}$ for large $n$. It follows that $y_n(\delta) \in [y_{n,-}, y_{n,+}]$ and $C D_n \rme^{-\alpha y_n(\delta)} \le \delta$ for all (sufficiently large) $n$.
Defining
\[
	u_n = 1 \vee \frac{\log D_n}{\log n}
\]
we have, as $n \to \infty$, the asymptotic relations [recall from the lines following \eqref{eq:chainyn+-} that $2 \log(2/\nu_n) \le \log n$ for large $n$]
\begin{equation*}
	y_{n,-} = \Oh( u_n^{3/2} \log n )
	\qquad \text{and} \qquad
	\log D_n = \Oh(u_n \log n).
\end{equation*}
Since $\nu_n \tilde{\tau}_n = \gamma_n^2 \vee (\tilde{\tau}_n b_n)$, we find
by \eqref{eq:major} that
%\[
%\supf \left| \sum_{1 \le i \ne j \le n} g_{n,f}(X_i, X_j) \right| 
%= \Op \left( n \nu_n \tau_n \left( y_{n,-} \vee (\log D_n) \right) \right),
%\qquad n \to \infty.
%\]
\begin{align}
\label{eq:term3:rate:init}
	\supf \left| \sum_{1 \le i \ne j \le n} g_{n,f}(X_i, X_j) \right| 
	&= \Op \left(
		n \nu_n \tilde{\tau}_n \left(y_{n,-} \vee \log D_n\right)
	\right) \\
\label{eq:term3:rate:init:2}
	&= \Op \left( n \left(\gamma_n^2 \vee (\tilde{\tau}_n b_n)\right) u_n^{3/2} \log n \right),
	\qquad n \to \infty.
\end{align}
Write
\begin{equation}
\label{eq:rn}
	\rho_n = u_n^{3/2} \frac{\log n}{n}
\end{equation}
and note that $b_n = \Oh(\rho_n^{1/2})$ as $n \to \infty$. It follows that
\begin{equation}
\label{eq:term3:rate}
	\supf \left| \frac{1}{n^2} \sum_{1 \le i \ne j \le n} g_{n,f}(X_i, X_j) \right| 
	= \Op \left( 
		\left\{ \gamma_n^2 \vee \left(\tilde{\tau}_n \rho_n^{1/2}\right) \right\} \rho_n 
	\right),
%	= \Op \left( \left(\gamma_n^2 r_n\right) \vee \left(G_n r_n^{3/2}\right) \right)
	\qquad n \to \infty.
\end{equation}

\paragraph{Step~4: Comparison and simplification of rates.} 

The first term in the expansion \eqref{eq:n2Pndecomp} was shown in Step~2 to have rate $ n \gamma_n^2 + L_n \tau_n \sqrt{n} \log(B_n) $. By definition, $n \gamma_n^2$ is of smaller order than the rate in \eqref{eq:term3:rate:init:2}. The other term, $L_n \tau_n \sqrt{n} \log(B_n)$, is of smaller order too: since $\nu_n \ge b_n$ and since $b_n$ in \eqref{eq:bn} is of larger order than $1/\sqrt{n}$, the rate $ L_n \tau_n \sqrt{n}  \log B_n  $ is of smaller order than $\Oh \left(n \nu_n L_n \tau_n \log B_n\right)$, which is bounded by the rate in \eqref{eq:term3:rate:init} in view of $L_n \tau_n   =  \tilde \tau_n$ and the connection between $B_n$ and $D_n$ in~\eqref{eq:Dn}.

We can therefore conclude that the rate in \eqref{eq:term3:rate} for the third term is actually the dominating one. It remains to work out the sequence $(\rho_n)_{n\in\N}$ in \eqref{eq:rn}, that is, to analyse $(u_n^{3/2})_{n\in\N}$ further. By \eqref{eq:Dn*},
\[
	\frac{\log D_n^*}{\log n} = \Oh \left( \left( \frac{n}{\log n} \right)^{2/3} \right), \qquad n \to \infty.
\]
Since $M_n$ is bounded by a constant multiple of $\ninf{q_n}^{1/2}$ (Lemma~\ref{lemma:maj M_n}), which grows at most at a polynomial rate in $n$ by Condition~\ref{cond:qn}, we have
\[
	\frac{\log\left( \ninf{q_n}^{1/2} / M_n \right)}{\log n}
	= \Oh \left( 1 + \frac{(\log M_n^{-1})_+}{\log n} \right), \qquad n \to \infty.
\]
In view of the connection between $D_n$ and $D_n^*$ in \eqref{eq:Dn*Dn}, it follows from the combination of the two estimates above that
\begin{align*}
	\left( \frac{\log D_n}{\log n} \right)^{3/2}
	&= \Oh \left( \frac{n}{\log n} \wedge \left\{ 1 + \frac{(\log M_n^{-1})_+}{\log n} \right\}^{3/2} \right) \\
	&= \Oh \left( \frac{n}{\log n} \wedge \left\{ 1 + \left( \frac{(\log M_n^{-1})_+}{\log n} \right)^{3/2} \right\} \right),
	\qquad n \to \infty.
\end{align*}
Since both members of the minimum on the right-hand side are larger than one, it follows that
\begin{align*}
	\rho_n = u_n^{3/2} \frac{\log n}{n} 
	&=
	\Oh \left(
		1 \wedge \left[ \frac{\log n}{n} \left\{ 1 + \left( \frac{(\log M_n^{-1})_+}{\log n} \right)^{3/2} \right\} \right] 
	\right) 
%	\\
%	&= 
%	\Oh \left(
%		1 \wedge \left( \frac{\log n}{n} + \frac{(\log M_n^{-1})_+^{3/2}}{n (\log n)^{1/2}} \right)
%	\right), 
	, \qquad n \to \infty.
\end{align*}
It is the latter form that is stated in the theorem. \qed

\paragraph{Proof of Corollary~\ref{cor:simple-rate}.}

If $(\log M_n^{-1})_+ = \Oh( \log n )$ as $n \to \infty$, then $M_n > a_n$ and thus $\tau_n = L_n^2$ for sufficiently large $n$. Moreover, $r_n$ can then be replaced by $(\log n) / n$. The simpler rate \eqref{eq:cor:simpler-rate} follows. Under the additional condition on $L_n^2$, the latter rate implies the one in \eqref{eq:cor:simple-rate}. \qed

\bigskip

\section{Proof of Proposition~\ref{prop:general term_Pn f}}
\label{supp:prop:general term_Pn f}

In this proof, we consider $ Z_n := \sup_{\scfunc\in\Sc} \left| P_n(\scfunc)-P(\scfunc) \right|$. Thanks to the triangle inequality, we get 
\begin{equation}
\label{eq:Zn:decomp} 
	Z_n \le \expec(Z_n) + \left| Z_n - \expec(Z_n) \right|. 
\end{equation}
We treat the two terms on the right-hand side of \eqref{eq:Zn:decomp} in Steps~1 and~2, respectively.  The bound for $Z_n$ then follows by adding both bounds in Step~3. 

\paragraph{Step~1: Expectation of the supremum.}

Let $(\eta_i)_i$ denote a sequence of independent Rademacher variables, that is, $\pr(\eta_i = +1) = \pr(\eta_i = -1) = 1/2$ for all $i$, and such that $(\eta_i)_i$ and $(X_i)_i$ are independent. 
The symmetrization inequality detailed in \citet[Lemma~2.3.1]{vdvaart+w:1996} gives
\[
n \expec(Z_n)
\le 2 \expec\left[ \sup_{\scfunc\in \Sc}\left|\sum_{i=1}^{n} \eta_{i} \{\scfunc(X_i) - P(\scfunc)\} \right| \right].
\]
Further, by applying Proposition~2.1 in \citet{gine2001consistency} to the class $\{\scfunc - P(\scfunc) \,:\, \scfunc\in \Sc\}$, (with VC parameters $(w, 3\sqrt{\rme} B)$, envelope $2U$ and variance bound $\tau^2$), we obtain the existence of a universal constant $C_1 >0$ such that
\begin{align}\label{eq:bound:expec_sup_diff}
\expec \left[ \sup_{\scfunc\in\Sc} \left|\sum_{i=1}^{n} \eta_{i} \{\scfunc(X_{i}) - P(\scfunc)\}  \right| \right]
\le C_1\left(2 w U \log \left( C_2 \theta \right) + \tau \sqrt{w n \log \left( C_2\theta \right)}\right) =: \xi_n
\end{align}
where $C_2 = 6\sqrt{\rme}$ and where we wrote $\theta =   B U / \tau$ as in the statement of the proposition.

\paragraph{Step 2: Concentration of the supremum around its expectation.}

Since $\sup_{s \in \Sc} \ninf{s - P(s)} \le 2U$, Theorem~1.4 in \citet{talagrand1996new} states the existence of a universal constant $K>0$ such that
\begin{equation}
\label{eq:Zndeviation:concentration}
	\forall t > 0, \qquad 
	\pr \left( n \abs{ Z_n - \expec(Z_n) } \ge t \right)
	\le K \exp \left\{
		-\frac{t}{2 K U} \log \left(1+\frac{2tU}{V_n}\right)
	\right\}
\end{equation}
where
\[
V_n = \expec\br{\sup_{\scfunc\in\Sc}\sum_{i=1}^{n}\{\scfunc(X_i)-P(\scfunc)\}^2}.
\]
Since $\log(1+x) \ge x / (1 + x/2)$ for all $x \ge 0$, we get
\begin{align*}
	\frac{t}{2 K U} \log \left(1+\frac{2tU}{V_n}\right)
	\ge \frac{t}{2 K U} \frac{\frac{2tU}{V_n}}{1 + \frac{1}{2} \frac{2tU}{V_n}}
	= \frac{t^2}{K \left(V_n + tU\right)}
\end{align*}
and thus,
\begin{equation*}
	\forall t > 0, \qquad 
	\pr \left( n\abs{ Z_n - \expec(Z_n) } \ge t   \right)
	\le K \exp \left\{
		-\frac{t^2 }{K \left(V_n + t  U\right)}
	\right\}.
\end{equation*}
Assuming without loss of generality that $K\geq 1$ and inverting the previous bound \citep[see for instance][Lemma~1]{peel2010empirical} gives that with probability at least $1-\delta$,
\begin{align*}
n\abs{ Z_n - \expec(Z_n) } \leq \sqrt { V_n K \log(K/\delta)} + UK \log( K/\delta).
\end{align*}
Corollary~3.4 in \citet{talagrand1994sharper} applied to the family $\ac{\scfunc-P(\scfunc): \scfunc\in\Sc}$ yields
\begin{align*}
	V_n 
	&\le n\tau^2+16U \expec\br{\sup_{\scfunc\in\Sc}\absbigg{\sum_{i=1}^{n}\eta_i\{\scfunc(X_i)-P(\scfunc)\}}} \\
	&\le n\tau^2+16 U \xi_n
\end{align*}
in view of \eqref{eq:bound:expec_sup_diff}.
Note that in the cited corollary, the functions are bounded in absolute value by $1$, whereas here, the functions $s - P(s)$ are bounded uniformly by $2 U$, and this is reflected in the bound above.
Since $\sqrt{a+b} \le \sqrt{a}+\sqrt{b}$ for nonnegative $a$ and $b$,
\[ 
	\sqrt{V_n} \le \tau\sqrt{n} + 4 \sqrt{U \xi_n}. 
\]
As a consequence, with probability at least $1-\delta$, it holds that
\begin{equation*}
n\abs{ Z_n - \expec(Z_n) } \leq  \left(\tau\sqrt{n} + 4 \sqrt{U \xi_n} \right) \sqrt { K \log(K/\delta)} + UK \log( K/\delta).
\end{equation*}

\paragraph{Step~3: Bound on the supremum.}

Combining the bound~\eqref{eq:Zn:decomp} with the inequalities obtained in Steps~1 and~2, we obtain that with probability at least $1-\delta$,
\begin{align*}
	n Z_n &\le 2 \xi_n   + 2 \times 2 \sqrt{ \xi_n}   \sqrt { U K \log(K/\delta)} + UK \log( K/\delta) + \tau  \sqrt { n K \log(K/\delta)}\\
	&\le 4 \xi_n + 3  UK \log( K/\delta) + \tau  \sqrt { n K \log(K/\delta)}
\end{align*}
where we have just used that $2ab\leq a^2 + b^2$ with $a = \sqrt{\xi_n}$ et $b = \sqrt{U K \log(K/\delta)}$. 
Injecting the value of $\xi_n$ from \eqref{eq:bound:expec_sup_diff} and factorizing, we get
\begin{align*}
	n Z_n 
	&\le  4 C_1 \left(2 w U \log ( C_2 \theta ) 
	+ \tau \sqrt{w n \log ( C_2 \theta )}\right)   + 3  UK \log( K/\delta) + \tau  \sqrt { n K \log(K/\delta)}\\
	&\leq U \left(\left(8 C_1 w\right)  \vee \left(3  K\right) \right) \log ( C_2 K \theta / \delta ) 
	+  \tau \sqrt{ n } \left(\left(4C_1\sqrt w\right)  \vee \sqrt K \right) 
	\left( \sqrt{ \log ( C_2 \theta )} + \sqrt{\log(K/\delta) } \right) \\
	&\leq U \left(\left(8 C_1 w\right)  \vee \left(3  K\right) \right) \log ( C_2 K \theta  / \delta ) +  
	\tau \sqrt{ n } \left(\left(4C_1\sqrt w\right)  \vee \sqrt K \right)  \sqrt{ 2 \log ( C_2 K \theta /\delta  )},
\end{align*}
where, in the last step, we have used that $ \sqrt a + \sqrt b \leq \sqrt { 2( a+ b)}$ with $a = \log(2\theta)$ and $b = \log(K/\delta)$. 
Conclude the proof by applying the inequality $(aw) \vee b \le (a \vee b) w$ for $w \ge 1$ and $a, b \ge 0$; similarly for $w$ replaced by $\sqrt{w}$.

\section{Proof of Theorem \ref{prop:oracle_bound}}\label{app:proofMC1}

Start by noting that 
\[
	\hat \alpha_{n,f}^{\mathrm{or}} - P(f)  = P_n  (\eps_{n,f}) . 
\]
As shown in Step~1 in Section~\ref{supp:prop:cover},  the residual class $\Ec_n$ is VC of parameters $\left( 2v  , A_n\right)$ with respect to the envelope $M_n$ with $A_n = 8 A \ninf{F} \ninf{q_n} ^{1/2} / {M_n} $. 
Hence we can apply Proposition~\ref{prop:general term_Pn f} with $w_n = (2v) \vee 1$, $B_n = 8 \left(A\vee (3/4) \sqrt {\rme} \right) \ninf{F} \ninf{q_n} ^{1/2} / {M_n} \ge A_n$, $ \tau_n = \sigma_n $ and $ U_n = M_n \vee  (2\sigma_n) $. Condition~(i) of Proposition~\ref{prop:general term_Pn f} is easily met. Because $M_n \leq 2 \ninf {F}  \ninf{q_n}^{1/2}$ we have
\begin{equation*}
B_n  \geq 4 \left(A\vee \left((3/4) \sqrt {\rme}\right) \right)  \geq 3\sqrt {\rme}  .
\end{equation*}
Therefore Condition~(ii) of Proposition~\ref{prop:general term_Pn f} is also met. Since $w_n$ is constant, we obtain
%( (8 (A\vee 3/4 \sqrt {\rme} ) \ninf{F} \ninf{q_n} ^{1/2} (M_n \vee  2\tau_n ) /M_n  ) \vee   (M_n \vee  2\tau_n ) )  / \tau_n$,
	\begin{equation*}
		\supf  \left| P_n  (\eps_{n,f})  \right|
		= \Op\left(\sigma_n  \sqrt{ n^{-1} \log(\theta_n )} + U_n n^{-1} \log(\theta_n) \right), \qquad n\to\infty,
	\end{equation*}	
	with $\theta_n = B_n \left(M_n \vee  \left(2\sigma_n\right) \right)  / \sigma_n  $. Note that $ \liminf_{n\to \infty } \theta_n >0$ and, since $M_n \vee (2\sigma_n) \le 2M_n$,
	\[
	\theta_n = \left( 8 \left(A\vee \left((3/4) \sqrt {\rme}\right) \right) \ninf{F} \ninf{q_n} ^{1/2}  \right)  \left(M_n \vee (2\sigma_n)\right) / \left(\sigma_n M_n\right) 
	\leq A' \ninf{q_n} ^{1/2}  / \sigma_n,  
	\] 
	with  $A' =  16 \left( A\vee \left((3/4) \sqrt {\rme}\right) \right) \ninf{F}   $. From Condition~\ref{cond:qn}, it holds that $\log(\ninf{q_n} ) = \Oh(\log(n) ) $.
Using the fact that $ M_n^2  = \Oh ( \sigma_n^2  n / \log(n )  ) $, which is also $ \Oh ( \sigma_n^2 n ) $, and since by assumption $  M_n^{-1}  = \Oh ( n^{\alpha} )  $, we get $ \sigma_n^{-2}  = \Oh( n^{ 1 + 2 \alpha}) $ as $n \to \infty$. Consequently, $\log(\theta_n ) = \Oh ( \log(n))$ and we find that 
	\[
	\supf  \left| P_n  (\eps_{n,f})  \right|
		= \Op\left(\sigma_n \sqrt{  n^{-1} \log( n )} + M_n n^{-1} \log(n) \right), \qquad n\to\infty.
	\]	
Moreover, using again that $M_n ^2= \Oh ( \sigma_n ^2 n / \log (n )   )$ as $n \to \infty$, we find the stated rate.
	%\begin{align*}
%\sup_{f\in \mathcal F}  \left| P_n  (\epsilon_{n,f})  \right| &= \Oh_{\pr}\left(\frac{ \supf P(\eps_{n,f}^2)   }{\sqrt{n}}\sqrt{ (2v) \log\left(\frac{8 A \ninf{F} \ninf{q_n} }{ \supf P(\eps_{n,f}^2)  }\right)}\right), \qquad n\to\infty.
%\end{align*}
\qed

\section{Proof of Theorem \ref{th:unif_cv}}\label{app:proofMC2}

Let $H_n = P(h_nh_n^\T)\in \reals^{(d_n+1)\times (d_n+1)} $ and $G_n = P(g_ng_n^\T)\in \reals^{d_n\times d_n}$. By assumption, both matrices are invertible. Using Equation~(22) in \cite{leluc2019control}, we obtain
\begin{equation*}
\left|\hat \alpha_{n,f}  - P(f) \right| 
\leq \left| P_n  (\eps_{n,f})  \right|  +   \left| G_n^{1/2} \left(\hat \beta_{n,f} - \beta_{n,f}\right) \right|_2 \left| G_n^{-1/2} P_n( g_n ) \right|_2.
\end{equation*}
Since $h_n = (1, g_n^\T)^\T$, we have 
\begin{equation*}
 \left|  H_n^{1/2} \begin{pmatrix}  \hat \alpha_{n,f} - \alpha_{n,f} \\ \hat \beta_{n,f} - \beta_{n,f} \end{pmatrix}   \right|_2 ^2  
 = \left( \hat \alpha_{n,f} - \alpha_{n,f} \right)^2 + \left| G_n^{1/2}  \left(\hat \beta_{n,f} - \beta_{n,f} \right)    \right|_2 ^2  ,
\end{equation*}
which, combined with the identity~\eqref{eq:Lfdecomp}, gives
\begin{align*}
 \left| G_n^{1/2}  \left(\hat \beta_{n,f} - \beta_{n,f} \right)    \right|_2 ^2  
 &\leq \left|  H_n^{1/2} \begin{pmatrix}  \hat \alpha_{n,f} - \alpha_{n,f} \\ \hat \beta_{n,f} - \beta_{n,f} \end{pmatrix}   \right|_2 ^2 \\
 &= L_f (\hat \alpha_{n,f} , \hat \beta_{n,f}) - L_f(  \alpha_{n,f},\beta_{n,f}) 
\end{align*}
with $ L_f(  \alpha , \beta )  = P [ ( f -  \alpha  -   \beta^{\T}h_n  )^2] $; note the slight change in notation of the excess risk due to the presence of an intercept. Consequently, we have shown that
\begin{equation*}
	\left| \hat \alpha _ {n,f} - P(f) \right| 
	\leq \left| P_n  (\eps _{n,f})\right| +  \sqrt { L_f(\hat \alpha_{n,f} , \hat \beta_{n,f}) - L_f(  \alpha_{n,f},\beta_{n,f})   }  \left| G_n^{-1/2} P_n( g_n )  \right|_2,
\end{equation*}
and the rest of the proof consists in bounding the three terms on the right-hand side uniformly in $f \in \Fc$. First, using Corollary~\ref{cor:simple-rate} and the fact that $\gamma_n^2  \leq \sigma_n^2 \ninf {q_n} $, we get the uniform bound
	\begin{equation*}		
	\sup_{f \in \Fc} \{ L_f(\hat \alpha_{n,f}, \hat{\beta}_{n,f}) - L_f(\alpha_{n,f}, \beta_{n,f}) \}
		= \Oh_{\pr}\left(   \sigma_n^2 \ninf {q_n}  \frac{ \log n}{n}  \right), \qquad n \to \infty.
	\end{equation*}
Second, using $\expec (| G_n^{-1/2} g_n|_2^2) = \operatorname{tr} ( G_n^{-1} G_n ) =  d_n$, we obtain
$\expec( | G_n^{-1/2} P_n( g_n )  |_2^2 ) = \Oh  (  { d_n/n}) $ and it follows by Markov's inequality that 
\[ 
	\left| G_n^{-1/2} P_n(g_n)  \right|_2 = \Op  \left(  \sqrt  { d_n/n} \right), \qquad n \to \infty . 
\] 
Third, because $L_n^2 = \Oh (\gamma_n ^2 \sqrt{ n / \log (n ) }) $ implies $ M_n^2  = \Oh ( \sigma_n^2 \sqrt {n / \log(n ) } ) $, which is also  $\Oh ( \sigma_n^2  n / \log(n )  )$, we can apply Theorem~\ref{prop:oracle_bound} to get
\[
	\supf  \left| P_n  (\eps_{n,f})  \right|
		= \Op\left(\sigma_n \sqrt{  n^{-1} \log( n )} \right), \qquad n\to\infty. 
\]
\qed

\section{A concentration inequality for degenerate U-statistics}
\label{supp:major}

The main term in the proof of Theorem~\ref{thm:main} concerned a degenerate U-statistic of order two, see Step~3 in Section~\ref{supp:thm:main}. We dealt with it via a special case of the concentration inequality in Theorem~2 in \citet{major2006estimate}, stated next.

\begin{theorem}[Special case of Theorem~2 in \citet{major2006estimate}]
	\label{prop:U}
	Let $(\Xc, \Ac, P)$ be a probability space and let $\Gc$ be an at most countably infinite collection of measurable functions $g : \Xc^2 \to [-1, 1]$ such that $\int_{\Xc} g(x, z) \, \diff P(z) = \int_{\Xc} g(z, x) \, \diff P(z) = 0$ for every $x \in \Xc$. Assume that $\Gc$ is a countable VC-class of parameters $(w, B)$, with $w \ge 1$ and $B > 0$. Let $\nu \in (0, 1]$ be such that $\sup_{g \in \Gc} \expec[g^2(X_1, X_2)] \le \nu^2$. Let $X_1, \ldots, X_n$ be an independent random sample from $P$. There exist universal positive constants $\alpha$, $C$ and $K$ such that
	\[
	\forall y \in [y_-, y_+], \qquad
	\pr \left( \sup_{g \in \Gc} \left| \sum_{1 \le i \ne j \le n} g(X_i, X_j) \right|
	\ge 2 n \nu y \right)
	\le
	C B^w e^{-\alpha y}
	\]
	where
	\begin{align*}
	y_- &= K \left[w + (w \log B / \log n)_+\right]^{3/2} \log(2/\nu), &
	y_+ &= n \nu^2.
	\end{align*}
\end{theorem}

\begin{ack}
The authors gratefully acknowledge comments and suggestions by an anonymous Reviewer that stimulated us to sharpen the main theorem and work out the application to Monte Carlo methods.
The authors also thank R\'emi Leluc for his valuable feedback and Aigerim Zhuman for her insightful remarks.
This project was supported financially by FNRS-F.R.S. grant CDR J.0146.19.
\end{ack}

\bibliographystyle{chicago}
\bibliography{biblio}

\end{document}